\documentclass{IEEEtran}
\usepackage[utf8]{inputenc}
\usepackage{amsmath}
\usepackage{amsthm}
\usepackage{amsfonts}
\usepackage{amssymb}
\usepackage{comment}
\usepackage{mathtools}
\usepackage{hyperref}
\usepackage{algorithm}
\usepackage{algpseudocode}
\usepackage{graphicx}
\usepackage{xcolor}
\usepackage{enumitem}
\usepackage{authblk}

\newtheorem{lemma}{Lemma}
\newtheorem{remark}{Remark}
\newtheorem{definition}{Definition}
\newtheorem{assumption}{Assumption}
\newtheorem{theorem}{Theorem}
\newtheorem*{exmp}{Example}

\newcommand{\N}{\mathbb{N}}
\newcommand{\D}{\mathcal{D}}
\newcommand{\R}{\mathbb{R}}
\newcommand{\C}{\mathcal{C}}
\DeclareMathOperator*{\argmin}{arg\,min}

\newcommand\blfootnote[1]{%
  \begingroup
  \renewcommand\thefootnote{}\footnote{#1}%
  \addtocounter{footnote}{-1}%
  \endgroup
}

\title{Gradient Based Clustering}

\graphicspath{{Images/}}

\author[1]{Aleksandar Armacki}
\author[2]{Dragana Bajovic}
\author[3]{Dusan Jakovetic}
\author[1]{Soummya Kar}

\affil[1]{\small Carnegie Mellon University, Pittsburgh, PA \authorcr Email: {\tt \{aarmacki, soummyak\}@andrew.cmu.edu}}
\affil[2]{\small Faculty of Technical Sciences, University of Novi Sad, Novi Sad, Serbia \authorcr Email: {\tt dbajovic@uns.ac.rs}}
\affil[3]{\small Faculty of Sciences, University of Novi Sad, Novi Sad, Serbia \authorcr Email: {\tt dusan.jakovetic@dmi.uns.ac.rs}}

\begin{document}

\maketitle

\begin{abstract}
We propose a general approach for distance based clustering, using the gradient of the cost function that measures clustering quality with respect to cluster assignments and cluster center positions. The approach is an iterative two step procedure (alternating between cluster assignment and cluster center updates) and is applicable to a wide range of functions, satisfying some mild assumptions. The main advantage of the proposed approach is a simple and computationally cheap update rule. Unlike previous methods that specialize to a specific formulation of the clustering problem, our approach is applicable to a wide range of costs, including non-Bregman clustering methods based on the Huber loss. We analyze the convergence of the proposed algorithm, and show that it converges to the set of appropriately defined fixed points, under arbitrary center initialization. In the special case of Bregman cost functions, the algorithm converges to the set of centroidal Voronoi partitions, which is consistent with prior works. Numerical experiments on real data demonstrate the effectiveness of the proposed method.
\end{abstract}

\section{Introduction}

Clustering is a fundamental problem in unsupervized learning and is ubiquitous in various applications and domains, \cite{banerjee_anomaly}, \cite{Huber_clust}, \cite{JAIN2010651}, \cite{dhillon2003adivisive}.
 $K$-means \cite{Lloyd} is a classical and widely adopted method for clustering. For a given target number $K$ of clusters, $K$-means proceeds iteratively by alternating between two steps: 1) cluster assignment, i.e., assign each data point to its closest (in terms of the Euclidean distance) cluster; and 2) finding cluster centers, i.e., position each cluster's center at the average of the data points currently assigned to the cluster. Besides $K$-means, popular clustering methods include its improved version $K$-means++ \cite{kmeans++}, as well as $K$-modes \cite{Huang97clusteringlarge}, $K$-medians \cite{arya-k-med}, \cite{arora_kmedians}, etc.\blfootnote{The work of A. Armacki and S. Kar was partially supported by the National
Science Foundation under grant CNS-1837607. The work of D. Bajovic and
D. Jakovetic is supported by the European Union’s Horizon 2020 Research
and Innovation program under grant agreements No 957337 and 871518. This
paper reflects only the authors’ views and the European Commission cannot be
held responsible for any use which may be made of the information contained
therein. Correspondence to: Aleksandar Armacki \tt{aarmacki@andrew.cmu.edu}.}
 
 It is well-known, e.g., \cite{4767478}, that $K$-means can be formulated as a joint minimization of a loss function with respect to two groups of variables: 1) binary variables that encode cluster assignments; and 2) continuous variables that designate cluster centers, where the corresponding loss function is a squared Euclidean norm. This $K$-means representation has motivated a class of new clustering methods called Bregman clustering \cite{JMLR:v6:banerjee05b}, where the squared Euclidean norm is replaced with arbitrary Bregman divergence \cite{BREGMAN1967200}, such as Kullback-Leibler, Mahalanobis, etc.
 An appealing feature of Bregman clustering is that the introduction of a different loss (other than squared Euclidean) does not harm computational efficiency, as, despite a more involved loss function, the cluster center finding step is still akin to $K$-means,  i.e., it corresponds to computing an average vector.
 
Several relevant clustering methods have been proposed that also generalize the squared Euclidean norm of $K$-means and that \emph{do not correspond to a Bregman divergence}. For example, clustering methods based on the Huber loss~\cite{huber_loss} have been shown to exhibit good clustering performance and exhibit a high degree of robustness to noisy data, \cite{Huber_clust}, \cite{liu2019robust}. However, several challenges emerge when generalizing clustering beyond Bregman divergences. First, the cluster center finding step--that corresponds to minimizing the loss with respect to cluster center variables--is no longer an average-finding operation and may be computationally expensive. Second, convergence and stability results for clustering beyond Bregman divergences are limited. For example, reference \cite{Huber_clust} shows a local convergence to a stationary point, assuming that the algorithm starts from an accurate cluster assignment.
 
In this paper, we propose a novel generalized clustering algorithm for a broad class of loss functions, and we provide a comprehensive convergence (stability) analysis for the algorithm. The assumed class of losses includes symmetric Bregman divergences (squared Euclidean norm, Mahalanobis, Jensen-Shannon, etc.), but more importantly, includes non-Bregman losses such as the Huber loss. The main novelty of the algorithm is that, at the cluster center finding step, the exact minimization of the loss function is replaced with a single gradient step with respect to the loss, hence significantly reducing computational cost in general. We prove that the algorithm converges to the appropriately defined stationary points associated with the joint loss with respect to the cluster assignment and cluster center variables, with arbitrary initialization. Numerical experiments on real data demonstrate that involving the cheap cluster center update incurs no or negligible loss both in clustering performance (appropriately measured accuracy) and in iteration-wise convergence speed, hence opening room for significant computational savings. We also show by simulation that the proposed method with the Huber loss exhibits a high degree of robustness to noisy data. While this is in line with prior findings on Huber-based clustering \cite{liu2019robust}, \cite{Huber_clust}, the proposed Huber-based method exhibits stronger theoretical convergence guarantees than those offered by the previous work. 

We now briefly review the literature to help us contrast the paper with existing work. Gradient based clustering has been explored in the context of the $K$-means cost in \cite{Macqueen67somemethods}, \cite{bottou_kmeans}. \cite{Macqueen67somemethods} analyzes a gradient based update rule for $K$-means, while \cite{bottou_kmeans} demonstrate that the standard centroid based solution of the $K$-means problem is equivalent to performing a Newton's method in each step. However, their analysis only concerns the squared Euclidean cost. Our work is considerably more general and can be applied to costs such as the Huber loss, or a class of Bregman divergences. \cite{Monath2017GradientbasedHC} propose a gradient-based approach for the problem of hierarchical clustering. \cite{paul2021uniform} use adaptive gradient methods to design a unified framework for robust center-based clustering, applicable to a large class of Bregman divergences.

A similar approach is used in the robotics community, in the context of coverage control problems, e.g. \cite{cortes}, \cite{Schwager09agradient}. However, the focus of their work is on continuous time gradient flow, designed for robot motion in a an environment that is typically an infinite set. Additionally, the authors in \cite{cortes} propose a family of discrete time algorithms, that converge to sets of centroidal Voronoi partitions, if the cost is squared Euclidean distance. On the other hand, our work focuses on a discrete time gradient algorithm, designed for clustering a finite set of points. We explicitly characterize the conditions under which the method converges, and extend the notion of distance to other metrics, beyond the Euclidean distance. 

\textbf{Paper organization}. The remainder of the paper is organized as follows. Section \ref{sec:problem} formally defines the clustering problem. Section \ref{sec:method} describes the proposed method. Section \ref{sec:analysis} presents the main results. Section \ref{sec:fixed_pt} presents an analysis of the fixed points the algorithm converges to. Section \ref{sec:num} presents numerical experiments, and Section \ref{sec:conclusion} concludes the paper. The appendix contains proofs of some technical results used throughout the paper.

\textbf{Notation}. $\R$ denotes the set of real numbers, while $\R^d$ denotes the corresponding $d$-dimensional vector space. More generally, for a vector space $V$, we denote by $V^K$ its $K$-dimensional extension. $\R_+$ denotes the set of non-negative real numbers. We denote by $\N$ the set of non-negative integers. $\| \cdot \|: \mathbb{R}^d \mapsto \mathbb{R}_+$ represents the standard Euclidean norm, while $\langle \cdot, \cdot \rangle: \R^d \times \R^d \mapsto \R$ denotes the inner product. $\nabla$ denotes the gradient operator, i.e., $\nabla_x f(x,y)$ denotes the gradient of the cost $f$ with respect to variable $x$. $[N]$ denotes the set of integers up to and including $N$, i.e., $[N] = \{1,\ldots,N\}$. In the algorithm description and throughout the analysis we use subscript to denote the iteration counter, while the value in the parenthesis corresponds to the particular center/cluster. In other words, $x_t(i)$ stands for the $i$-th cluster center at iteration $t$. Same holds for clusters, i.e., $C_t(i)$ denotes the $i$-th cluster at iteration $t$, corresponding to the subset of the data points assigned to cluster $i$, at iteration $t$.

\section{Problem formulation}\label{sec:problem}

In this section we formalize the clustering problem, and propose a general cost, that subsumes many of the previous clustering formulations.

Let $(\R^d,g)$ represent the standard $d$-dimensional real vector space, and a corresponding distance function. Let $\D \subset \R^d$ be a finite set, with an associated probability measure $\mu_{\D}$. For some $K > 1$, the problem of clustering the points in $\D$ into $K$ clusters can be cast as
\begin{equation}\label{eq:Kmeans}
    \min_{x \in \R^{K d}} \sum_{y \in \D}p_y \min_{i \in [K]} g(x(i), y)^2,
\end{equation}
where $x = \begin{bmatrix}x(1)^T,\ldots,x(K)^T \end{bmatrix} \in \R^{K d}$ represent the candidate cluster centers and $p_y \in (0,1)$\footnote{Note that, while a standard probability measure can take values in $[0,1]$,  we implicitly assume two things: the support of $\mu_{\D}$ is the whole set $\D$, and $\D$ contains at least two distinct points.}, given by $p_y \coloneqq \mu_{\D}(y)$, represent problem independent weights, that measure the importance of data points $y \in \D$. In the case when $g$ is the standard Euclidean distance, (\ref{eq:Kmeans}) is known in the literature as the $K$-means problem \cite{awasthi2014center}. Another problem similar in nature to (\ref{eq:Kmeans}) is given by
\begin{equation}\label{eq:Kmedians}
    \min_{x \in \R^{K d}}\sum_{y \in \D}p_y \min_{i \in [K]}g(x(i), y),
\end{equation} and for $g$ being the Euclidean distance, is known in the literature as $K$-medians \cite{arora_kmedians}. Both problems have been well studied, and are known to be NP-hard~\cite{Vattani2010TheHO}, \cite{awasthi2015hardness}, \cite{Kmedians-NP}. Many algorithms for solving (\ref{eq:Kmeans}) and (\ref{eq:Kmedians}) exist, guaranteeing convergence to locally optimal solutions, e.g. \cite{Lloyd}, \cite{Macqueen67somemethods}, \cite{JMLR:v6:banerjee05b}, \cite{pmlr-v9-telgarsky10a}, \cite{arora_kmedians}, \cite{arya-k-med}. However, all of the algorithms are specialized for solving either the $K$-means or the $K$-medians problem, and hence are not generally applicable. 

The problems (\ref{eq:Kmeans}), (\ref{eq:Kmedians}), can be equivalently defined as follows. For any $K > 1$, we call $C = (C(1),\ldots,C(K))$ a partition of $\D$, if $\D = \cup_{i \in [K]}C(i)$ and $C(i) \cap C(j) = \emptyset, \: \text{for } i \neq j$. Denote by $\C_{K,\D}$ the set of all $K$-partitions of $\D$. The clustering problem (\ref{eq:Kmeans}) is then equivalent to
\begin{equation}\label{eq:Kmeans2}
    \min_{x \in \R^{K d}, C \in \C_{K,\D}}J(x,C) = \sum_{i \in [K]}\sum_{y \in C(i)}p_yg(x(i),y)^2.
\end{equation} The problem (\ref{eq:Kmedians}) can be defined in the same way. 

We propose to unify and generalize (\ref{eq:Kmeans}) and (\ref{eq:Kmedians}) as follows. Let $f: \R^d \times \R^d \mapsto \R_+$, be a loss function that satisfies the following assumption. 

\begin{assumption}\label{asmpt:g&f}
The loss function $f$ is increasing with respect to the function $g$, i.e., for all $x, y, z \in \mathbb{R}^d$
\begin{equation*}
    g(x,y) \leq g(z,y) \text{ implies } f(x,y) \leq f(z,y).
\end{equation*}
\end{assumption}

We can then define the following general problem 
\begin{equation}\label{eq:cost}
       \min_{x \in R^{K d}, C \in \C_{K,\D}} J(x,C) = \sum_{i = 1}^K\sum_{y \in C(i)}p_yf(x(i),y).
\end{equation} 

\begin{remark}
Introducing the function $f$ along with $g$ allows us to naturally decouple the concepts of \emph{cluster shape} and \emph{location of cluster center}. In particular, the function $g$ dictates the cluster shape, while the choice of function $f$ determines the exact location of the cluster centers. We elaborate further on this in Section \ref{sec:fixed_pt}.
\end{remark}

\begin{remark}
Compared to (\ref{eq:Kmeans2}), the formulation (\ref{eq:cost}) is more general, in the sense that, while the dependence of $f$ on $g$ is maintained, via Assumption \ref{asmpt:g&f}, the function $f$ provides more flexibility, as is illustrated by the following examples.   
\end{remark}

\begin{exmp} For the choice $g(x,y) = \|x - y\|$, and $f(x,y) = g(x,y)$, the $K$-medians formulation is recovered. For the choice $g(x,y) = \|x - y\|$, and $f(x,y) = g(x,y)^2$, the $K$-means formulation is recovered. For the choice $g(x,y) = f(x,y)$, being a Bregman distance, the Bregman divergence clustering formulation from \cite{JMLR:v6:banerjee05b} is recovered. For the choice $g(x,y) = \|x - y\|$, and $f(x,y) = \phi_{\delta}(g(x,y))$,  where $\phi_{\delta}(x)$ is the Huber loss, the formulation from \cite{Huber_clust} is recovered. We recall that the Huber loss is defined by
\begin{equation}\label{eq:Huber}
    \phi_{\delta}(x) = \begin{cases}
        \frac{x^2}{2}, &|x| \leq \delta \\
        \delta |x| - \frac{\delta^2}{2}, &|x| > \delta
    \end{cases}.
\end{equation} 
\end{exmp}

\section{The proposed method}\label{sec:method}

In this section we outline the proposed method for solving instances of (\ref{eq:cost}) that satisfy some mild assumptions (see ahead Assumptions \ref{asmpt:met}-\ref{asmpt:co-coerc}).

To solve (\ref{eq:cost}), an iterative approach is proposed. Starting from an arbitrary initialization $x_0$, at every iteration $t$, it maintains and updates the pair $(x_t,C_t)$, where $x_t \coloneqq [x_t(1)^T,
x_t(2)^T,
\ldots ,
x_t(K)^T
]^T \in R^{K d}$ and $C_t \coloneqq (C_t(1),\ldots,C_t(K))$ represent stacks of centers and clusters at time $t \in \N$. The iterative approach consists of two steps:
\begin{enumerate}
    \item \textit{Cluster reassignment}: for each $y \in \D$, we find the index $i \in [K]$, such that
    \begin{equation}\label{eq:reassign}
        g(x_t(i),y) \leq g(x_t(j),y), \forall j \ne i,
    \end{equation} and assign the point $y$ to cluster $C_{t+1}(i)$. 
    \item \textit{Center update}: for each $i \in [K]$, we perform the following update
    \begin{equation}\label{eq:grad_local}
        x_{t+1}(i) = x_t(i) - \alpha \sum_{y \in C_{t+1}(i)}p_y\nabla_x f\big(x_t(i),y\big),
    \end{equation} where $\alpha > 0$ is a fixed step-size.
\end{enumerate} Note that (\ref{eq:grad_local}) can be written compactly as
\begin{equation}\label{eq:grad_global}
    x_{t+1} = x_t - \alpha \nabla_x J(x_t,C_{t+1}),
\end{equation} where $\nabla_x J(x_t, C_{t+1}) \in \R^{Kd}$ is the gradient of $J$ with respect to $x$, whose $i$-th block of size $d$ is given by
\begin{equation}\label{eq:grad_J}
    \big[\nabla_x J(x_t,C_{t+1})\big]_i = \sum_{y \in C_{t+1}(i)}p_y\nabla_x f(x_t(i),y). 
\end{equation}

In addition to Assumption \ref{asmpt:g&f}, for our method to be applicable, we make the following assumptions on functions $f$, $g$ and $J$.

\begin{assumption}\label{asmpt:met}
The distance function $g$ is a metric, i.e., it satisfies the following properties: $1) \: g(x,y) \geq 0$, and $g(x,y) = 0 \iff x = y;$ $2) \: g(x,y) = g(y,x);$ $3) \: g(x,y) \leq g(x,z) + g(z,y)$.
\end{assumption}

\begin{remark}
Assumption \ref{asmpt:met} requires the distance function $g$, that dictates cluster assignment, to be a distance metric. Note that, with respect to \cite{JMLR:v6:banerjee05b}, Bregman divergences are not necessarily symmetric, nor do they obey the triangle inequality. However, \cite{bregman_triangle}, \cite{chen_bregman_metrics} show that a large class of Bregman divergences, such as Mahalanobis distances, as well as Jensen-Shannon divergence, represent squares of metrics. Hence, for the choice $f(x,y)$ a Bregman divergence representing the square of a metric and $g(x,y) = \sqrt{f(x,y)}$, Assumption \ref{asmpt:met} is satisfied. 
\end{remark}
    
\begin{assumption}\label{asmpt:coerc} 
The cost function $f$ is coercive with respect to the first argument, i.e.
$\lim_{\|x\| \rightarrow +\infty}f(x,y) = +\infty, \: \forall \: y \ne x$.
\end{assumption}    

\begin{remark}
Assumption \ref{asmpt:coerc} ensures that the sequence of centers, $\{x_t\}$, generated by (\ref{eq:grad_global}) remains bounded. It does so, by not allowing for $x$ to grow infinitely without affecting the loss function $f$.  
\end{remark}    

\begin{assumption}\label{asmpt:co-coerc} The function $J$ has co-coercive gradients in the first argument, i.e., for all $x,z \in \R^{Kd}$
\begin{align*}
    \langle \nabla_x J(x,C) -& \nabla_z J(z,C), x - z  \rangle \\ &\geq \frac{1}{L}\|\nabla_x J(x,C) - \nabla_z J(z,C) \|^2.
\end{align*}
\end{assumption}

\begin{remark}
Assumption \ref{asmpt:co-coerc} ensures that the sequence of centers, $\{x_t\}$, generated by (\ref{eq:grad_global}) not only decreases the cost $J$, but also decreases the distance of the generated sequence $\{x_t\}$ to a stationary point $x_*$ (or the set of stationary points in general), at every iteration. 
\end{remark}

\begin{remark}\label{rmk:Lischitz-grad}
Assumption \ref{asmpt:co-coerc} implies Lipschitz continuos gradients with respect to the first argument of the function $J$, as a result of the Cauchy-Schwartz inequality. As we show in the Appendix, Assumption \ref{asmpt:co-coerc} is satisfied for any function that is convex and has Lipschitz continuous gradients.
\end{remark}

\begin{remark}
Note that Assumption \ref{asmpt:co-coerc} rules out non-smooth costs, such as $K$-medians, (\ref{eq:Kmedians}). However, when a desirable feature of the cost is robustness, smooth costs like the Huber loss can be used.
\end{remark}

\section{Convergence analysis}\label{sec:analysis}

In this section the goal is to show that the method (\ref{eq:reassign})-(\ref{eq:grad_local}) converges to a fixed point.

To begin with, the notions of a fixed point and a set of optimal clusterings are defined.

\begin{definition}\label{def:fix-pt}
The pair $(x_*,C_*)$ is a fixed point of the clustering procedure (\ref{eq:reassign})-(\ref{eq:grad_local}), if the following holds:
\begin{enumerate}
    \item Optimal clustering with respect to centers: for each $i \in [K]$, and each $y \in C_*(i)$, we have
    \begin{equation}\label{eq:opt_clust}
        g(x_*(i),y) \leq g(x_*(j),y), \forall j \neq i.
    \end{equation}
    \item Optimal centers with respect to clustering:
    \begin{equation*}
        \nabla_x J(x_*,C_*) = 0.
    \end{equation*}
\end{enumerate}
\end{definition}

\begin{definition}\label{def:U}
Let $x \in \R^{Kd}$ represent cluster centers. We say $U_x$ is the set of optimal clusterings with respect to $x$, if for all clusterings $C \in U_x$, (\ref{eq:reassign}) is satisfied.
\end{definition}

\begin{definition}\label{def:Ubar}
Let $x \in \R^{Kd}$ represent cluster centers. We define the set $\overline{U}_x$ as the set of clusterings with respect to $x$ such that: $1) \: \overline{U}_x \subset U_x$; $2) \: \forall C \in \overline{U}_x: \: \nabla_x J(x,C) = 0$. 
\end{definition}

\begin{remark}
As we show in Section \ref{sec:analysis}, for a Bregman cost (of which the $K$-means problem is a special case) any fixed point, per Definition $\ref{def:fix-pt}$, represents a centroidal partition of the data, i.e., the centers $x_*(i)$ correspond to the means of clusters $C_*(i)$. This is consistent with results in \cite{JMLR:v6:banerjee05b}, and shows that Definition \ref{def:fix-pt} is a natural one.
\end{remark}

\begin{remark}
In a slight abuse of terminology, we will refer to a point $x$ as fixed point, if there exists a clustering $C$ such that $(x,C)$ satisfies Definition \ref{def:fix-pt}.
\end{remark}

\begin{remark}
Note that, by Definition \ref{def:Ubar}, a pair $(x,C)$ is a fixed point if $C \in \overline{U}_x$. 
\end{remark}

The main result of the paper is stated in Theorem \ref{thm:convergence}, which shows the convergence of the sequence of cluster centers to a fixed point. 

\begin{theorem}\label{thm:convergence}
Let Assumptions \ref{asmpt:g&f}, \ref{asmpt:met}, \ref{asmpt:coerc}, \ref{asmpt:co-coerc} hold. For the step-size choice $\alpha < \frac{2}{L}$ and any $x_0 \in \R^{K d}$, the sequence of centers $\{x_t\}$ generated by the algorithm (\ref{eq:reassign})-(\ref{eq:grad_local}), converges to a fixed point $x_* \in \R^{K d}$, i.e., a point such that $\overline{U}_{x_*} \neq \emptyset$.  
\end{theorem}

The result of Theorem \ref{thm:convergence} is strong - for a fixed step-size, under arbitrary initialization, the proposed algorithm converges to a fixed point. In the context of $K$-means clustering, e.g. \cite{Lloyd}, \cite{JMLR:v6:banerjee05b}, we achieve the same guarantees. In the context of different costs, e.g. Huber loss, compared to \cite{Huber_clust}, where the authors show convergence of the sequence of centers, under the assumptions that the clusters have already converged, and the initialization $x_0$ is sufficiently close to a fixed point, our results are much stronger - we guarantee that the full sequence $\{x_t\}$ converges to a fixed point, under arbitrary initialization. We also show that the clusters converge.  

To prove Theorem \ref{thm:convergence}, a series of intermediate lemmas is introduced. The proof outline follows a similar idea as the one developed in~\cite{kar2019clustering}.

The following lemma shows that the proposed algorithm decreases the objective function $J$ in each iteration.

\begin{lemma}\label{lm:decr}
For the sequence $\{(x_t,C_t)\}$, generated by (\ref{eq:reassign})-(\ref{eq:grad_local}), with $\alpha < \frac{2}{L}$, the resulting sequence of costs $\{J(x_t,C_t)\}$ is non-increasing.
\end{lemma}
\begin{proof}
To begin with, note that (\ref{eq:reassign}) together with Assumption \ref{asmpt:g&f} implies that the clustering reassignment step decreases the cost, i.e.
\begin{align}\label{eq:leq1}
    \begin{aligned}
        J(x_t,C_{t+1}) &= \sum_{i = 1}^K\sum_{y \in C_{t+1}(i)}p_yf\big(x_t(i),y\big) \\ &\leq \sum_{i = 1}^K\sum_{y \in C_t(i)}p_yf\big(x_t(i),y\big) = J(x_t,C_t).
    \end{aligned}
\end{align} Next, using Lipschitz continuity of gradients of $J$ (recall Remark \ref{rmk:Lischitz-grad}), we have
\begin{align*}
    J(x_{t+1},C_{t+1}) &\leq J(x_t,C_{t+1}) + \frac{L}{2}\|x_{t+1} - x_t \|^2 \\ &+ \Big\langle \nabla_x J(x_t,C_{t+1}), x_{t+1} - x_t \Big\rangle.
\end{align*} Using (\ref{eq:grad_global}), we get
\begin{align*}
    \begin{aligned}
        J(x_{t+1},C_{t+1}) \leq J(x_t,C_{t+1}) - c(\alpha) \| \nabla_x J(x_t,C_{t+1})\|^2,
    \end{aligned} 
\end{align*} where $c(\alpha) = \alpha \Big(1 - \frac{\alpha L}{2} \Big)$.  Choosing $\alpha < \frac{2}{L}$ ensures that $c(\alpha) > 0$, and combining with (\ref{eq:leq1}), we get
\begin{align}\label{eq:J-decr}
    \begin{aligned}
        J(x_{t+1},C_{t+1}) &\leq J(x_t,C_{t+1}) - c(\alpha)\| \nabla_x J(x_t,C_{t+1})\|^2 \\ &\leq J(x_t,C_t) - c(\alpha)\| \nabla_x J(x_t,C_{t+1})\|^2 \\ &\leq J(x_t,C_t),
    \end{aligned}
\end{align} which completes the proof.
\end{proof}

The following lemma shows that, if two cluster centers are sufficiently close, the optimal clustering sets match.

\begin{lemma}\label{lm:clust_convg}
Let $x \in \mathbb{R}^{K d}$ represent cluster centers. Then, $\exists \epsilon > 0$, such that, for any center $x^{\prime} \in \mathbb{R}^{K d}$, satisfying $\max_{i \in [K]} g(x(i),x^{\prime}(i)) < \epsilon$, we have $U_{x^{\prime}} \subset U_{x}$.
\end{lemma}
\begin{proof}
 For given cluster centers $x \in \mathbb{R}^{Kd}$ and each data point $y \in \mathcal{D}$, we denote by $\mathcal K_x^\star(y)$ the set of cluster indices $i$ whose centers $x(i)$ are closest to $y$:
\begin{equation*}
  \mathcal{K}_x^\star(y)=\argmin_{i \in [K]} g(x(i),y).   
\end{equation*}
Define 
\begin{equation}
\label{eq:best-off}
\epsilon_0:=  \min_{y \in \D} \min_{i \in [K]\setminus \mathcal K_x^\star(y)}  g(x(i),y) - g(x({c^\star_y}), y),    
\end{equation}
where $c^\star_y$ denotes an arbitrary cluster in  $\mathcal K_x^\star(y)$. By the construction of $\mathcal K_x^\star(y)$ and finiteness of the set of data points $\D$, we have that $\epsilon_0>0$.

Let $\mathcal X_{x,\epsilon}:= \left\{ x^\prime \in \mathbb{R}^{K d}: g (x(i),x(i)^\prime) < \epsilon, \forall i \in [K] \right\} $, where $\epsilon>0$. We show that, for each $x^\prime \in \mathcal X_{x,\epsilon_0/2}$,  for each $y \in \D$, there holds 
\begin{equation}
\label{eq:nested}
  \mathcal{K}_{x^\prime}^\star(y)\subseteq \mathcal{K}_x^\star(y).  
\end{equation}
From~\eqref{eq:nested}, it is easy to see that any optimal cluster assignment with respect to $x^\prime$, $C \in U_{x^\prime}$, will also be optimal with respect to $x$, thus implying the claim of the lemma. 

To prove~\eqref{eq:nested}, fix an arbitrary data point $y$ and fix an arbitrary $ i \in \mathcal K_{x^\prime}^\star (y) $. We want to show that $i\in \mathcal K_{x}^\star (y)$ as well, i.e., that cluster center $x(i)$ belongs to the set of cluster centers $x$  closest to $y$. By the triangle inequality for $g$, we have 
\begin{align}
\begin{aligned}
\label{eq:below-epsilon_0}
g(x(i), y) &\leq g(x(i), x^\prime(i)) + g(x^\prime(i), y) \\ &< \frac{\epsilon_0}{2} + g(x^\prime(j),y) \\ &\leq \frac{\epsilon_0}{2} + g(x(j), x^\prime(j))+ g(x(j), y) \\ &< \epsilon_0 + g(x(j), y),
\end{aligned}
\end{align} where in the second line we use the fact that $x^\prime \in \mathcal X_{x, \epsilon_0/2}$ (for index $i$) and the fact that $i\in \mathcal K_{x^\prime}^\star (y)$, in the third line we apply the triangle inequality for $g$, and in the fourth line we use again the fact that $x^\prime$ is in the $\epsilon_0/2$ neighborhood of $x$ (for index $j$). 
For the sake of contradiction, suppose now that $ i\notin \mathcal K_x^\star (y)$ and take $j \in \mathcal K_x^\star (y)$ (note that~\eqref{eq:below-epsilon_0} holds for all $j\in [K]$).  Then, from~\eqref{eq:best-off} we have $g(x(i),y) \geq g(x(j),y )+\epsilon_0$, which clearly contradicts~\eqref{eq:below-epsilon_0}. This proves~\eqref{eq:nested} and subsequently proves the lemma.    
\end{proof}

The next lemma shows that, if a limit point of the sequence of centers exists, it must be a fixed point.

\begin{lemma}\label{lm:fix_pt}
Any convergent subsequence of the sequence $\{x_t \}$, generated by (\ref{eq:reassign})-(\ref{eq:grad_local}), converges to a fixed point. 
\end{lemma}
\begin{proof}
Let $\{x_{t_s} \}_{s=0}^\infty$ be a convergent subsequence of $\{x_t \}$. Let $x_*$ be its limit point and assume the contrary, that $x_*$ is not a fixed point. By Definition $\ref{def:fix-pt}$, this means
\begin{equation*}
    \|\nabla_x J(x_*,C)\| > 0, \: \forall C \in U_{x_*}.
\end{equation*} As the number of possible clusterings is finite, we can define
\begin{equation}\label{eq:nonz-grad}
    \min_{C \in U_{x_*}}\| \nabla_x J(x_*,C) \| = \epsilon_1 > 0.
\end{equation}From the assumption $x_{t_s} \rightarrow x_*$, we have that, for a fixed $\delta_* > 0$, there exists a sufficiently large $s_0 > 0$, such that
\begin{equation*}
    \forall i \in [K], \: \forall s \geq s_0: \|x_{t_s}(i) - x_*(i)\| < \delta_*. 
\end{equation*} It then follows from the continuity of $g$ that there exists a sufficiently large $s_0 > 0$, such that $g(x_{t_s}(i),x_*(i)) < \epsilon_*$. Per Lemma \ref{lm:clust_convg}, we then have $C_{x_{t_s}+1} \in U_{x_{t_s}} \subset U_{x_*}$, $\forall s \geq s_0$. From (\ref{eq:nonz-grad}), we have
\begin{equation}\label{eq:contra}
    \| \nabla_x J(x_*,C_{t_s+1})\| \geq \epsilon_1, \: \forall s \geq s_0. 
\end{equation} Next, using the results established in Lemma \ref{lm:decr}, we have 
\begin{align*}
    \begin{aligned}
        J(x_{t+1}&,C_{t+1}) \leq J(x_t,C_t) - c(\alpha)\|\nabla_x J(x_t,C_{t+1}) \|^2 \\ &\leq J(x_{t-1},C_{t-1}) - c(\alpha)\|\nabla_x J(x_{t-1},C_t) \|^2 \\ &-c(\alpha) \|\nabla_x J(x_t,C_{t+1})\|^2 \leq \ldots \\ &\leq J(x_0,C_1) - c(\alpha)\sum_{r = 0}^t\|\nabla_x J(x_r,C_{r+1}) \|^2.
    \end{aligned}
\end{align*} Rearranging, we get
\begin{align}\label{eq:sum}
\begin{aligned}
    c(\alpha)\sum_{r = 0}^t\|\nabla_x J(x_r,C_{r+1}) \|^2 &\leq J(x_0,C_1) - J(x_{t+1},C_{t+1}) \\ &\leq J(x_0,C_1).
\end{aligned}
\end{align} Additionally, we have
\begin{equation}\label{eq:sum_subseq}
    \sum_{j=0}^{s(t)}\|\nabla_x J(x_{t_j},C_{t_j+1}) \|^2 \leq \sum_{j = 0}^t\|\nabla_x J(x_j,C_{j+1}) \|^2, 
\end{equation} where $s(t) = \sup\{j: t_j \leq t \}$. Combining (\ref{eq:sum}) and (\ref{eq:sum_subseq}), we get
\begin{equation}\label{eq:sum_subseq2}
    c(\alpha)\sum_{j=0}^{s(t)}\|\nabla_x J(x_{t_j},C_{t_j+1}) \|^2 \leq J(x_0,C_1).
\end{equation} Noting that the term on the right hand side of (\ref{eq:sum_subseq2}) is finite and independent of $t$, and $s(t) \rightarrow +\infty$ as $t \rightarrow +\infty$, we can take the limit as $t \rightarrow +\infty$, to obtain
\begin{align*}
    c(\alpha)\sum_{j = 0}^\infty \|\nabla_x J(x_{t_j},C_{t_j+1}) \|^2 \leq J(x_0,C_1) < +\infty, 
\end{align*} which implies
\begin{equation*}
    \lim_{s \rightarrow \infty} \|\nabla_x J(x_{t_s},C_{t_s+1}) \|^2 = 0.
\end{equation*} Fix an $\epsilon > 0$. By the definition of limits, there exists a $s_1 > 0$, such that
\begin{equation*}
    \|\nabla_x J(x_{t_s},C_{t_s+1}) \| < \epsilon, \: \forall s \geq s_1.
\end{equation*} On the other hand, from $x_{t_s} \rightarrow x_*$, there exists a $s_2 > 0$, such that
\begin{equation*}
    \|x_{t_s} - x_* \| < \epsilon, \: \forall s \geq s_2.
\end{equation*} As $C_{x_{t_s+1}} \in U_{x_{t_s}} \subset U_{x_*}$, $\forall s \geq s_0$, for any $s \geq \max\{s_0,s_1,s_2\}$, we have
\begin{align*}
    \|\nabla_x J&(x_*,C_{t_s+1}) \| \leq \|\nabla_x J(x_*,C_{t_s+1}) - \nabla_x J(x_{t_s},C_{t_s+1}) \| \\ &+ \| \nabla_x J(x_{t_s},C_{t_s+1}) \| \leq L \| x_* - x_{t_s} \| + \epsilon < (L + 1)\epsilon,
\end{align*} where we used the Lipschitz continuity of the gradients of $J$ in the second inequality. As $\epsilon > 0$ was arbitrarily chosen, we can conclude 
\begin{equation}
\label{eq:grad-limit}
    \| \nabla_x J(x_*,C_{t_s+1}) \| \rightarrow 0,
\end{equation} which clearly contradicts (\ref{eq:contra}). Hence, we can conclude that $x_*$ is a fixed point, i.e., 
\begin{equation*}
    \exists C \in U_{x_*}: \: \|\nabla_x J(x_*,C)\| = 0.
\end{equation*}
\end{proof}

The next lemma proves a stronger result, namely, that the clusters converge in finite time.

\begin{lemma}\label{lm:Clust_fin_conv}
For any convergent subsequence of the sequence of centers, $\exists s_0 > 0$, such that $\forall s \geq s_0: \: C_{t_s+1} \in \overline{U}_{x_*}$, where $x_*$ is the limit of the sequence $\{x_{t_s} \}$.
\end{lemma}
\begin{proof}
Let
\begin{equation*}
\delta:=\min_{C\in U_{x^\star}\setminus \overline{U}_{x^\star}} \|\nabla_x J(x^\star,C) \|.    
\end{equation*}
Note that, by construction of  $\overline{U}_{x^\star}$, it must be that $\|\nabla_x J(x^\star,C) \|>0$ for each $C\in U_{x^\star}\setminus \overline{U}_{x^\star}$, which together with the finiteness of $U_{x^\star}\setminus \overline{U}_{x^\star}$, implies $\delta>0$. 

For the sake of contradiction, suppose now that $C_{t_s+1} \in U_{x^\star}\setminus \overline{U}_{x^\star}$, infinitely often. Then, $\|\nabla_x J(x^\star,C_{t_s+1})\| \geq \delta$ infinitely often, which clearly contradicts~\eqref{eq:grad-limit}.   
\end{proof} 

The following lemma shows that the generated sequence of cluster centers stays bounded. 

\begin{lemma}\label{lm:bdd-seq}
The sequence of cluster centers $\{x_t \}$, generated by (\ref{eq:reassign})-(\ref{eq:grad_local}), is bounded.
\end{lemma}
\begin{proof}
By Lemma \ref{lm:decr}, we have
\begin{align}\label{eq:contr-setup}
\begin{aligned}
    J(x_{t+1},C_{t+1}) &\leq J(x_t,C_{t+1}) \leq \ldots \leq J(x_1,C_1) \\ &\leq J(x_0,C_1) < +\infty.
\end{aligned}
\end{align}
Recalling equation (\ref{eq:cost}), for $x\in \mathbb R^{K d}$ and a clustering $C$, we define
\begin{equation*}
    J_i(x(i),C(i)) = \sum_{y \in C(i)}p_y f(x(i),y),
\end{equation*} so that 
\begin{equation}\label{eq:J_sum}
    J(x,C) = \sum_{i = 1}^K J_i(x(i),C(i)).
\end{equation}

For the sake of contradiction, suppose that the sequence of centers $\{x_t\}$ is unbounded. This implies the existence of a cluster $k$ and a subsequence $t_s$ such that $\|x_{t_s}(i)\| \rightarrow +\infty$. For each $t_s$, let $\underline{t_s} = \max\{t \leq t_s: \: C_t(i) \ne \emptyset \}$, i.e., $\underline{t_s}$ is the largest element in the sequence prior to $t_s$, such that the $i$-th cluster is non-empty.   

Recalling the update rule (\ref{eq:grad_local}), it is not hard to see that $x_{t_s}(i) = x_{\underline{t_s}}(i)$, for all $s$, implying $\|x_{\underline{t_s}}(i)\| \rightarrow +\infty$.
By Assumption \ref{asmpt:coerc} and the fact that $C_{\underline{t_{s}}}(i)$ is nonempty for each $s$, we have
\begin{align}\label{eq:J_i-coerc}
    \begin{aligned}
        \lim_{\|x_{\underline{t_s}}(i) \| \rightarrow +\infty} J_k(x_{\underline{t_s}}(i),C_{\underline{t_s}}(i)) 
        =  +\infty.
    \end{aligned}
\end{align} Note that this is the case regardless of the clustering $C_{t_s}$, as the dataset $\D$ is finite, and therefore a bounded set. It is easy to see that unboundness of $J_i$ implies unboundedness of $J$, i.e.,  $\lim_{s \rightarrow +\infty} J(x_{\underline{t_s}}, C_{\underline{t_s}})=+\infty$. But this contradicts~\eqref{eq:contr-setup}, hence proving the claim of the lemma.   
\end{proof}

The next lemma shows that, if a point in the sequence of centers is sufficiently close to a fixed point, then all the subsequent points remain in the neighborhood of the fixed point.

\begin{lemma}\label{lm:final-step}
Let $\{x_t\}$ be the sequence of cluster centers generated by (\ref{eq:reassign})-(\ref{eq:grad_local}), with the step-size satisfying $\alpha < \frac{2}{L}$. Let $x_*$ be a fixed point, in the sense that $\overline{U}_{x_*} \ne \emptyset$. Then, $\exists \epsilon_{x_*} > 0$, such that, $\forall \epsilon \in (0,\epsilon_{x_*})$, $\exists t_{\epsilon} > 0$, such that, if $\|x_{t_0} - x_* \| \leq \epsilon,$ for some $t_0 > t_{\epsilon}$, then $\|x_t - x_* \| \leq \epsilon$, for all $t \geq t_0$.
\end{lemma} 
\begin{proof}
Recall that, by Lemma \ref{lm:decr}, the sequence of costs $\{J(x_t,C_t) \}_{t \geq 0}$ is decreasing. Moreover, since $J(x,C) \geq 0$, we know that the limit of the sequence of costs exists and is finite. Let
\begin{equation}\label{eq:J*def}
    J_* = \lim_{t \rightarrow \infty}J(x_t,C_t).
\end{equation} By assumption, $\overline{U}_{x_*} \ne \emptyset$. From the definition of $\overline{U}_{x_*}$, for all $C \in U_{x_*} \setminus \overline{U}_{x_*}$ we have
\begin{equation}\label{eq:grad-nonz}
    \|\nabla_x J(x_*,C) \| > 0.
\end{equation} As $U_{x_*}$ is a finite set, we can define
\begin{equation*}
    \epsilon_1 = \min_{C \in U_{x_*} \setminus \overline{U}_{x_*}} \|\nabla_x J(x_*,C)\| > 0.
\end{equation*} Let $\epsilon_* > 0$ be such that Lemma \ref{lm:clust_convg} holds. From the continuity of $g$, we have
\begin{equation}\label{eq:delta}
    \exists \delta_* > 0 \: \forall x: \|x - x_*\| < \delta_* \implies g(x,x_*) < \epsilon_* .
\end{equation} Define
\begin{equation}\label{eq:eps*}
    \epsilon_{x_*} = \min\bigg\{\delta_*, \frac{\epsilon_1}{L} \bigg\}. 
\end{equation} For an arbitrary $\epsilon \in (0,\epsilon_{x_*})$, let $t_0 > 0$ be such that 
\begin{equation}\label{eq:condition}
    J(x_t,C_t) \leq J_* + \frac{c(\alpha)}{2}(\epsilon_1 - L\epsilon)^2, \: \forall t \geq t_0,
\end{equation} with $c(\alpha)$ defined as in Lemma \ref{lm:decr}. Note that the choice of $t_0$ is possible, from (\ref{eq:J*def}) and the fact that $(\epsilon_1 - L\epsilon)^2 > 0$. Our goal now is to show that, for a fixed $\epsilon\in (0,\epsilon_{x_*})$, if for some $t: \: t \geq t_0$ and $\|x_t - x_*\| < \epsilon$, then $\|x_{t+1} - x_* \| < \epsilon$.

First note that, if $t \geq t_0$ and $\|x_t - x_*\| < \epsilon$, it holds that $C_{t+1} \in \overline{U}_{x_*}$. To see this, assume the contrary, $\| x_t - x_* \| < \epsilon$ and $C_{t+1} \notin \overline{U}_{x_*}$. It follows from (\ref{eq:eps*}) that
\begin{equation*}
    \| x_t - x_* \| < \delta_*.
\end{equation*} From (\ref{eq:delta}) and Lemma \ref{lm:clust_convg}, we then have $U_{x_t} \subset U_{x_*}$, and hence, $C_{t+1} \in U_{x_*}$. Using Lipschitz continuity of gradients of $J$, we get
\begin{equation}\label{eq:step1}
    \|\nabla_x J(x_t,C_{t+1}) - \nabla_x J(x_*,C_{t+1}) \| \leq L\|x_t - x_*\| \leq L\epsilon.
\end{equation} As $C_{t+1} \notin \overline{U}_{x_*}$, from (\ref{eq:grad-nonz}), we have
\begin{equation}\label{eq:step2}
    \|\nabla_x J(x_*,C_{t+1}) \| \geq \epsilon_1.
\end{equation} Applying the triangle inequality, (\ref{eq:step1}) and (\ref{eq:step2}), we get 
\begin{equation}\label{eq:above}
    \|\nabla_x J(x_t,C_{t+1}) \| \geq \epsilon_1 - L\epsilon.
\end{equation} Note that, by (\ref{eq:eps*}), the right-hand side of (\ref{eq:above}) is positive. Combining (\ref{eq:J-decr}), (\ref{eq:condition}) and (\ref{eq:above}), we have
\begin{align*}
    \begin{aligned}
        J(x_{t+1},C_{t+1}) &\leq J(x_t,C_t) - c(\alpha)\|\nabla_x J(x_t,C_{t+1}) \|^2 \\ &\leq J_* + \frac{c(\alpha)}{2}(\epsilon_1 - L\epsilon)^2 - c(\alpha)\| \nabla_x J(x_t,C_{t+1}) \|^2 \\ &\leq J_* + \frac{c(\alpha)}{2}(\epsilon_1 - L\epsilon)^2 - c(\alpha)(\epsilon_1 - L\epsilon)^2 \\ &< J_*, 
    \end{aligned} 
\end{align*} which is a contradiction. Hence, $C_{t+1} \in \overline{U}_{x_*}$. 

Using Assumption \ref{asmpt:co-coerc}, the update rule (\ref{eq:grad_global}), and the fact that $C_{t+1} \in \overline{U}_{x_*}$, we have
\begin{align}\label{eq:interest}
    \begin{aligned}
        \|x_{t+1} - x_* \|^2 &= \|x_t - \alpha \nabla_x J(x_t,C_{t+1}) - x_* \|^2 \\ &= \|x_t - x_* \|^2 + \alpha^2\|\nabla_x J(x_t,C_{t+1}) \|^2 \\ &- 2\alpha\langle \nabla_x J(x_t,C_{t+1}), x_t - x_* \rangle \\ &\leq \|x_t - x_* \|^2 - \alpha\Big(\frac{2}{L} - \alpha \Big)\|\nabla_x J(x_t,C_{t+1}) \|^2 \\ &\leq \|x_t - x_*\|^2 < \epsilon^2,
    \end{aligned}
\end{align} where the second inequality follows from the step-size choice $\alpha < \frac{2}{L}$. Therefore, we have shown that
\begin{equation*}
    \| x_t - x_* \| < \epsilon \implies \| x_{t+1} - x_* \| < \epsilon.
\end{equation*} The same result holds for all $s > t$ inductively, which proves the claim.
\end{proof}

We are now ready to prove Theorem~\ref{thm:convergence}.

\begin{proof}[Proof of Theorem~\ref{thm:convergence}]
By Lemma \ref{lm:decr} and the fact that the corresponding sequence of costs $\{J(x_t,C_t)\}$ is nonnegative, we know this sequence converges to some $J_* \in \mathbb{R}_+$, by the monotone convergence theorem. On the other hand, by Bolzano-Weierstrass theorem and Lemma \ref{lm:bdd-seq}, the sequence $\{x_t\}$ has a convergent subsequence, $\{x_{t_s}\}$, with some $x_* \in \mathbb{R}^{Kd}$ as its limit. From the continuity of $J$ and convergence of $\{x_{t_s}\}$, we can then conclude that $J_* = \lim_{s \rightarrow +\infty}J(x_{t_s},C_{t_s}) = J(x_*, C_*)$, for some $C_* \in U_{x_*}$. Lemma \ref{lm:fix_pt} then implies that $x_*$ is a fixed point. Finally, Lemmas \ref{lm:Clust_fin_conv} and \ref{lm:final-step} imply the convergence of the entire sequence $\{x_t\}$ to $x_*$.
\end{proof}

\begin{remark}
We note that the convergence guarantees of our method are independent of the initialization. Therefore, our method is amenable to seeding procedures, such as $K$-means++.
\end{remark}

\section{Fixed point analysis}\label{sec:fixed_pt}

In this section we analyse the fixed points and their properties. To begin with, we formally define the notion of Voronoi partitions, e.g., \cite{Okabe2000SpatialTC}.

\begin{definition}\label{def:Voronoi}
Let $(V,d)$ be a metric space. For a set $X \subset V$, and $z = (z(1),\ldots,z(K)) \in V^K$, we say that $P = (P(1),\ldots,P(K))$ is a Voronoi partition of the set $X$, generated by $z$, with respect to the metric $d$, if $P$ is a partition of $X$ and additionally, for every $i \in [K]$
\begin{equation*}
    P(i) = \left\{x \in X: d(z(i),x) \leq d(z(j),x), \: \forall j \neq i \right\}.
\end{equation*}
\end{definition}

From Definitions \ref{def:fix-pt} and \ref{def:Voronoi}, it is clear that, for a fixed point $(x_*,C_*)$, the clustering $C_*$ represents a Voronoi partition of $\D$, with respect to $g$, generated by $x_*$. Moreover, from Definition \ref{def:U}, it is clear that, for any point $x$, the set $U_x$ represents the set of all possible Voronoi partitions of $\D$, generated by $x$. 

From the cluster reassignment step (\ref{eq:reassign}), we can see that in our approach, the clusters represent Voronoi partitions with respect to $g$. It is known that different distance metrics induce different Voronoi partitions, e.g., \cite{Okabe2000SpatialTC}, and the choice of metrics affects the shape of the resulting partitions. For example, choosing $g_1(x,y) = \|x - y\|$, the standard Euclidean distance and $g_2(x,y) = \|x - y\|_A$, a Mahalanobis distance (see (\ref{eq:Mahalanobis}) ahead), would potentially result in different Voronoi partitions of the dataset. In that sense, the distance function $g$ determines the \textit{cluster shape}.   

Using (\ref{eq:grad_J}), the fixed point condition from Definition \ref{def:fix-pt} is equivalent to
\begin{align}\label{eq:zero-grad-comp}
    \begin{aligned}
        &\forall i \in [K]: \nabla_x J_i(x_*(i),C_*(i)) = 0 \iff \\ &\forall i \in [K]: \sum_{y \in C(i)}p_y\nabla_x f(x_*(i),y) = 0.
    \end{aligned}
\end{align} From (\ref{eq:zero-grad-comp}), we can see that the exact location of a cluster center is determined by $f$. In that sense, the cost function $f$ determines the \textit{location of cluster centers}. For example, for the choice $g(x,y) = \|x - y\|$, $f_1(x,y) = \frac{1}{2}\|x - y\|^2$ and $f_2(x,y) = \phi_{\delta}(g(x,y))$, where $\phi_{\delta}$ is the Huber loss defined in (\ref{eq:Huber}), we can see that in both cases the cluster shapes will be determined by the Euclidean distance metric. However, applying (\ref{eq:zero-grad-comp}) to $f_1$ and $f_2$, it can be shown that
\begin{align*}
    x_1(i) &= \frac{1}{\mu_{\D}(C_1(i))}\sum_{y \in C_1(i)}p_y y, \\
    x_2(i) &= \frac{\sum_{y \in \underline{C_2}(i)}p_y y + \sum_{y \in \overline{C_2}(i)}\frac{\delta}{\|x_2(i) - y\|}p_y y}{\sum_{y \in \underline{C_2}(i)}p_y + \sum_{y \in \overline{C_2}(i)}\frac{\delta}{\|x_2(i) - y\|}p_y},
\end{align*} where $\overline{C_2}(i) = \{y \in C_2(i): \|x_2(i) - y\| > \delta \} $, $\underline{C_2}(i) = \{y \in C_2(i): \| x_2(i) - y \| \leq \delta \}$, $x_1(i)$ and $x_2(i)$ satisfy (\ref{eq:zero-grad-comp}) for $f_1$ and $f_2$ respectively, and $\mu_{\D}(C(i)) = \sum_{y \in C(i)}p_y$, represents the measure of the $i$-th cluster. Hence, we see that the function $f$ dictates the exact location of the cluster center within the cluster. 

\begin{remark}
Note that, while a fixed point of Huber loss takes the form of $x_2(i)$, as defined above, it is not actually a trivially computable closed form solution, as both sides of the equality contain $x_2(i)$. Therefore, to obtain such a form in practice, an iterative solver is required. 
\end{remark}

\subsection{Case study: Centroidal Voronoi Partitions}

A Voronoi partition $C$ of the set $\D$ generated by $x$ is called centroidal, if the generator of each partition corresponds to its center, i.e.
\begin{equation*}
    x(i) = \frac{1}{\mu_{\D}(C(i))}\sum_{y \in C(i)}p_yy, \: \forall i \in [K].
\end{equation*} The authors in \cite{JMLR:v6:banerjee05b} show that, if the cost function $f$ is a Bregman divergence, the Lloyd-type algorithm \cite{Lloyd} is optimal, i.e., using centroidal Voronoi partitions results in the minimal loss in Bregman information. In what follows, we show that, for a Bregman divergence-type cost function, our algorithm converges to the set of centroidal Voronoi partitions. To this end, we first define the notion of Bregman divergence.  

\begin{definition}\label{def:Bregman}
Let $\phi: \mathbb{R}^d \mapsto \mathbb{R}$ be a strictly convex, differentiable function. The Bregman divergence defined by $\phi$ is given by $d_{\phi}(p,q) = \phi(p) - \phi(q) - \langle \nabla \phi(q), p-q \rangle$.
\end{definition}

As a consequence of strict convexity of $\phi$, we have $d_{\phi} \geq 0$, and $d_{\phi}(p,q) = 0 \iff p = q$. However, in general, $d_{\phi}$ is not a metric. Therefore, in our framework, Bregman divergences are used as $f(x,y) = d_{\phi}(y,x)$. To define an appropriate metric $g$, we rely on the works \cite{bregman_triangle}, \cite{chen_bregman_metrics}, that show a rich class of Bregman divergences that represent squares of metrics. Examples include Mahalanobis distance based Bregman divergences, as well as the Jensen-Shannon entropy. We show in the Appendix that, on a properly defined support, the Jensen-Shannon entropy satisfies Assumptions \ref{asmpt:g&f}-\ref{asmpt:co-coerc}. Here, we define the Mahalanobis distance based Bregman divergences and show how they fit our framework. Let $A \in \mathbb{R}^{d \times d}$ be a symmetric positive definite matrix. The corresponding Bregman divergence is then given by 
\begin{equation}\label{eq:Mahalanobis}
    d_{\phi}(x,y) = \frac{1}{2}(x-y)^TA(x-y).
\end{equation} This class of Bregman divergences is covered by our formulation, for the choice
\begin{align*} 
    f(x,y) &= \frac{1}{2}(x-y)^TA(x-y), \\
    g(x,y) &= \|x - y\|_A,
\end{align*} where $\|x\|_A \coloneqq \sqrt{\langle Ax, x \rangle}$.

\begin{lemma}\label{lm:Bregman}
Let f be a Bregman divergence, satisfying Assumptions \ref{asmpt:g&f}-\ref{asmpt:co-coerc}. Then, the gradient clustering algorithm converges to the set of centroidal Voronoi partitions. 
\end{lemma}
\begin{proof}
To this end, we want to show that, for an arbitrary fixed point $(x_*,C_*)$ of the algorithm, the pair produces a centroidal Voronoi partition.

From Definition \ref{def:fix-pt}, it is clear that $C_*$ is a Voronoi partition of the dataset, generated by $x_*$. Now, let $f(x,y)$ be a Bregman divergence, for some strictly convex $\phi$. From the definition of Bregman divergence, we then have
\begin{align*}
    \nabla_x f(x,y) &= -\nabla \phi(x) + \nabla \phi(x) - \nabla^2 \phi(x)(y-x) \\ &= \nabla^2 \phi(x)(x - y). 
\end{align*} Combining with (\ref{eq:zero-grad-comp}), we get, for all $i \in [K]$
\begin{align*}
        0 &= \sum_{y \in C_*(i)}p_y\nabla_x f(x_*(i),y) \\ &= \nabla^2 \phi(x_*(i))\bigg(\sum_{y \in C_*(i)}p_y(x_*(i) - y)\bigg).
\end{align*} From the strict convexity of $\phi$, we have 
\begin{align}\label{eq:Bregm_fixed_pt}
    \begin{aligned}
        &\sum_{y \in C_*(i)}p_y\nabla_x f(x_*(i),y) = 0 \\ &\iff \sum_{y \in C_*(i)}p_y(x_*(i) - y) = 0 \\ &\iff  x_*(i) = \frac{1}{\mu_{\D}(C_*(i))}\sum_{y \in C_*(i)}p_yy.
    \end{aligned}
\end{align} We have shown that the generators of Voronoi partitions correspond to their respective centers, which completes the proof.
\end{proof}

\subsection{Case study: Beyond Centroidal Voronoi Partitions}

Note that, in the case the cost used is a Bregman distance, the fixed point has a closed-form solution (\ref{eq:Bregm_fixed_pt}). Therefore, in each iteration of the algorithm, it is possible to compute the optimal cluster center, which is exactly what the Lloyd algorithm does. The Lloyd algorithm \cite{Lloyd}, and its generalization \cite{JMLR:v6:banerjee05b}, perform the following two steps:
\begin{enumerate}
    \item \textit{Cluster reassignment}: for each $y \in \D$, find the cluster center $i \in [K]$, such that
    \begin{equation*}
        g(x_t(i),y) \leq g(x_t(j),y), \forall j \neq i,
    \end{equation*} and assign the point $y$ to cluster $C_{t+1}(i)$.
    \item \textit{Center update}: for each $i \in [K]$, perform the following update
    \begin{equation}\label{eq:Lloyd}
        x_{t+1}(i) = \frac{1}{\mu_{\D}(C_{t+1}(i))}\sum_{y \in C_{t+1}(i)}p_y y.
    \end{equation}
\end{enumerate}

The authors in \cite{bottou_kmeans} analyze the update rule (\ref{eq:Lloyd}) and show that it corresponds to performing a Newton step in each iteration. The authors in \cite{JMLR:v6:banerjee05b} show an even stronger result - in the case $f$ is a Bregman divergence, the update (\ref{eq:Lloyd}) corresponds to the optimal update, in terms of minimizing the Bregman information.  

From that perspective, naively extending the Lloyd's algorithm to a general cost $f$ would correspond to
\begin{equation}\label{eq:exact_min}
    x_{t+1}(i) = \argmin_{x(i)} \sum_{y \in C_{t+1}(i)}\frac{p_y}{\mu_\D(C_{t+1}(i))} f\big(x(i),y\big).
\end{equation} Performing the update (\ref{eq:exact_min}) would require solving an optimization problem in each iteration. This computation might be prohibitively expensive. In this case, the update (\ref{eq:grad_local}) is preferred, as computing the gradient is a feasible, and in many cases cheap operation. 

An example of such a function is the Huber loss, defined in (\ref{eq:Huber}). Huber loss provides robustness, e.g., \cite{Ke2005RobustL1}, \cite{liu2019robust}, as it behaves like the squared loss for points whose modulus is smaller than a given threshold, while it grows only linearly for points whose modulus is beyond the threshold. Therefore, Huber loss implicitly gives more weight to points with smaller modulus.  

In our framework, Huber loss is used as
\begin{equation}\label{eq:Huber2}
    f(x,y) = \phi_{\delta}(\|x - y\|) = \begin{cases}
        \frac{1}{2}\|x - y\|^2, &\|x-y\| \leq \delta \\
        \delta \|x-y\| - \frac{\delta^2}{2}, &\|x-y\| > \delta
    \end{cases}.
\end{equation} A closed form expression satisfying (\ref{eq:exact_min}), for the cost (\ref{eq:Huber2}) does not exist. Therefore, to perform the update (\ref{eq:exact_min}) in practice, requires solving an optimization problem in every iteration. On the other hand, from (\ref{eq:Huber}) and (\ref{eq:Huber2}), we have 
\begin{align*}
    \nabla_x f(x,y) = \begin{cases}
        (x-y), &\|x - y\| \leq \delta \\
        \delta\frac{x-y}{\|x-y\|}, &\|x-y\| > \delta
    \end{cases},
\end{align*} hence the gradient update is straightforward to compute. Note that computing the gradient update of the Huber loss corresponds to performing gradient clipping, effectively dampening the contribution of points that are far away from the current center estimate. We show in the Appendix that Huber loss satisfies Assumptions \ref{asmpt:g&f}-\ref{asmpt:co-coerc}.

\section{Numerical experiments}\label{sec:num}

In this section we demonstrate the effectiveness of the proposed method. The experiments presented in this section were performed on the MNIST \cite{lecun-mnisthandwrittendigit-2010} and Iris \cite{iris} datasets. Throughout the experiments, we assume a uniform distribution over the data, i.e., $\mu_{\D}(y_i) = \frac{1}{N}, \: \forall i = 1,\ldots,N$, with $\D = \{y_1,\ldots,y_N\}$.

The MNIST training dataset consists of handwritten digits, along with the corresponding labels. The data is initially normalized (divided by the highest value in the dataset), so that each pixel belongs to the $[0,1]$ interval. Next, we select the first 500 samples of the digits $1$ through $7$. In total, our dataset consists of $N = 3500$ points, each being in $[0,1]^{768}$ (as there are $28 \times 28$ pixels), with the number of underlying clusters $K = 7$. The Iris dataset consists of three species of the \emph{Iris} flower, \emph{Iris setosa}, \emph{Iris virginica} and \emph{Iris versicolor}, along with the corresponding labels. Each of the species has 50 samples, so that the total number of samples is 150. Each sample consists of 4 features, being the length and the width of the sepals and petals of the flowers. In total, the dataset consists of $N = 150$ points, with the number of underlying clusters $K = 3$.

For the first experiment, we utilised the gradient based clustering using the standard squared Euclidean cost. In our setup, that corresponds to:
$f(x,y) = \frac{1}{2}\|x - y\|^2$, $g(x,y) = \|x - y \|$. We refer to the resulting method as gradient $K$-means and compare it with the standard $K$-means \cite{Lloyd}, \cite{JMLR:v6:banerjee05b}. We set the step-size equal to $\alpha = \frac{1}{N}$, which results in $\alpha = \frac{1}{3500}$ for MNIST and $\alpha = \frac{1}{150}$ the Iris experiments. For a fair comparison, we set the initial centers of both methods to be the same. In particular, we take a random point from each class  and set them as the initial centroids. 

We run the clustering experiments for 20 times and present the mean performance (solid line), as well as the standard deviation (shaded region). The measure of performance used is the fraction of correctly clustered samples. Note that both methods are unsupervised, i.e., do not use labels when learning. However, we used the labels as ground truth, when comparing the clustering results. In order to account for a possible label mismatch, we checked all the possible label permutations when computing the clustering accuracy and chose the highest score as the true score. The results for MNIST and Iris datasets are presented in Figures \ref{fig:baseline} and \ref{fig:iris_kmean}, respectively. 

\begin{figure}[htp]
    \centering
    \includegraphics[width=8cm]{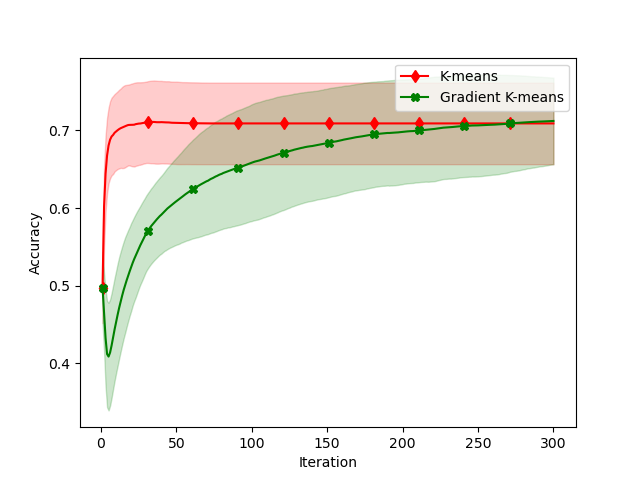}
    \caption{Accuracy of the Lloyd based $K$-means vs the gradient based $K$-means algorithm. Presents the accuracy of clustering digits $1$ through $7$ from the MNIST dataset.}
    \label{fig:baseline}
\end{figure}

\begin{figure}[htp]
    \centering
    \includegraphics[width=8cm]{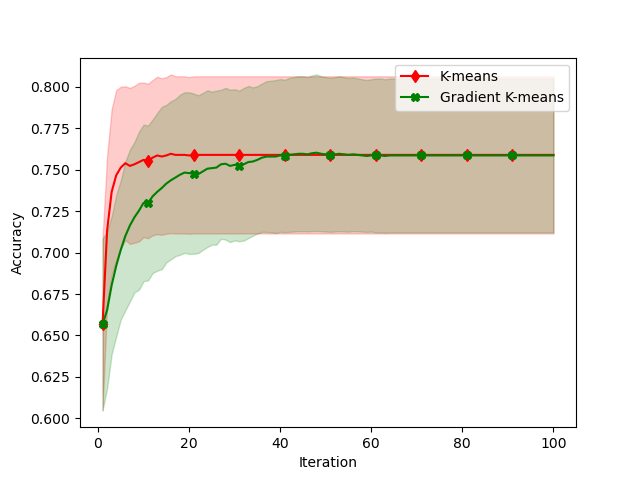}
    \caption{Accuracy of the Lloyd based $K$-means vs the gradient based $K$-means algorithm. Presents the accuracy of clustering flowers from the Iris dataset.}
    \label{fig:iris_kmean}
\end{figure}

Figure~\ref{fig:baseline} shows that accuracy-wise, the gradient based $K$-means slightly outperforms the standard $K$-means. Speed-wise, the standard $K$-means update converges faster, which is to be expected, as the $K$-means update corresponds to performing the exact $\argmin$ step in each iteration. Figure~\ref{fig:iris_kmean} shows that accuracy-wise, the gradient based $K$-means performs identically to the standard $K$-means, at a negligible speed loss. 

For the second experiment, we added zero mean Gaussian noise to a fraction of points from all classes, thus introducing noise. In order to combat the noise, we use a Huber loss function for our gradient based clustering method. 
In our framework, the Huber loss is used as in (\ref{eq:Huber2}). We compare the performance of the gradient based Huber loss clustering and the Huber based method from~\cite{Huber_clust}. The authors in~\cite{Huber_clust} consider a method that is based on a fixed-point iteration, given by the recursion 
\begin{align*}
    x_{t+1}(i) &= \frac{\sum_{y \in \underline{C_t}(i)}p_y y + \sum_{y \in \overline{C_t}(i)}\frac{\delta}{\|x_t(i) - y\|}p_y y}{\sum_{y \in \underline{C_t}(i)}p_y + \sum_{y \in \overline{C_t}(i)}\frac{\delta}{\|x_t(i) - y\|}p_y},
\end{align*} where $\overline{C_t}(i) = \{y \in C_t(i): \|x_t(i) - y\| > \delta \} $, $\underline{C_t}(i) = \{y \in C_t(i): \| x_t(i) - y \| \leq \delta \}$. The authors also suggest initializing the method by doing one round of Lloyd's algorithm from a random starting point. For fairness of comparison, we initialize both the gradient Huber and the method from~\cite{Huber_clust} (which we refer to as "Huber" in the figures) in this way.  

As in the previous experiment, we report the average results over 20 runs, along with the standard deviation. We consider the effects of changing the percentage of noisy samples and changing the variance of the noise. In all the experiments, we fix the Huber loss parameter to $\delta = 10$ for MNIST and $\delta = 5$ for the Iris dataset. We use the same step-size as in the standard $K$-means case, i.e., $\alpha = \frac{1}{N}$. The results for MNIST and Iris datasets are presented in Figures~\ref{fig:perc} and~\ref{fig:iris_perc}, respectively. 

\begin{figure}[ht]
\centering
\begin{tabular}{ll}
\includegraphics[scale=0.24]{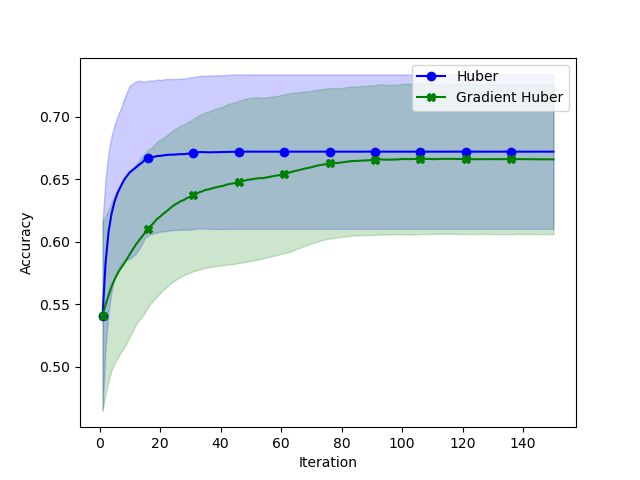}
&
\includegraphics[scale=0.24]{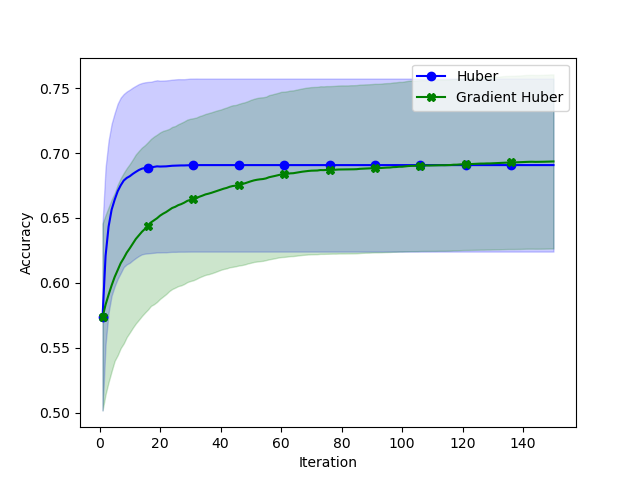}
\\
\includegraphics[scale=0.24]{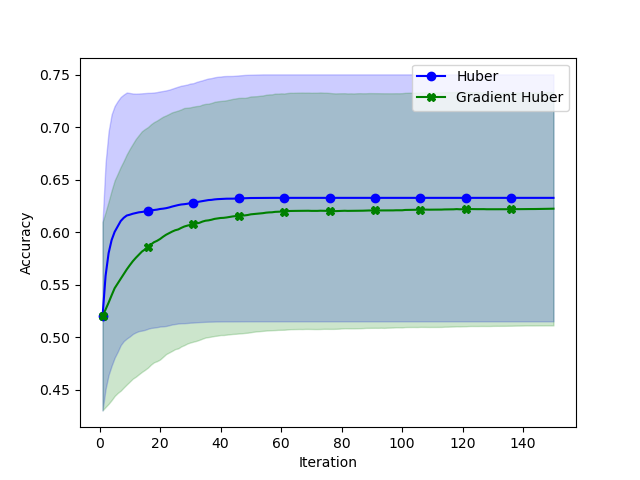}
&
\includegraphics[scale=0.24]{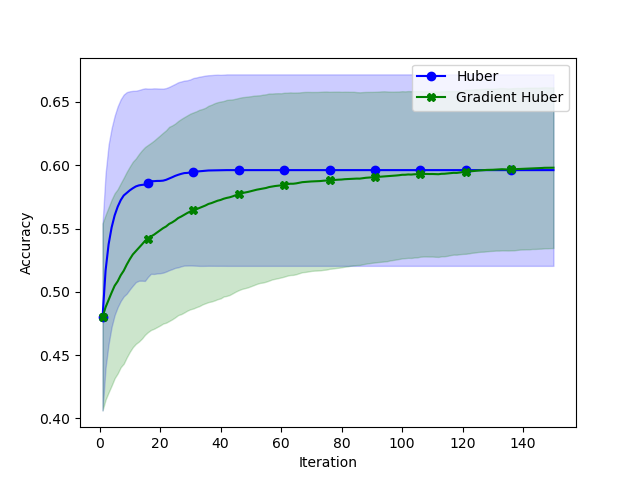}
\end{tabular}
\caption{Performance of Huber loss gradient vs the method from~\cite{Huber_clust} on MNIST data. The rows correspond to percentage of noisy samples being $10\%$ and $20\%$, with columns corresponding to variance of noise being $1$ and $2$, respectively (e.g., the upper left image corresponds to $10\%$ of noisy samples, with variance $1$, etc.).}
\label{fig:perc}
\end{figure}

\begin{figure}[ht]
\centering
\begin{tabular}{ll}
\includegraphics[scale=0.24]{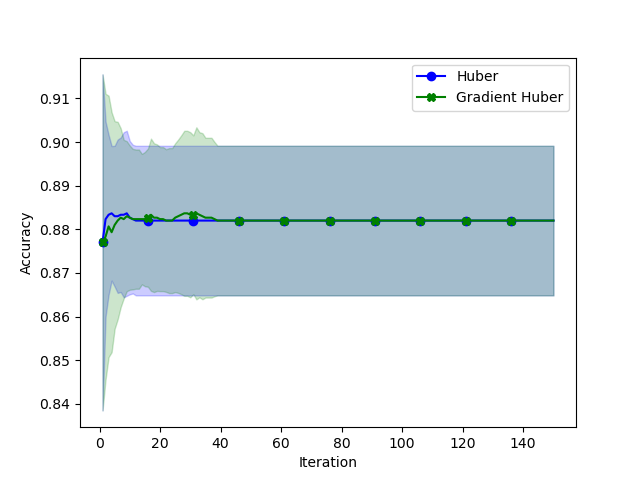}
&
\includegraphics[scale=0.24]{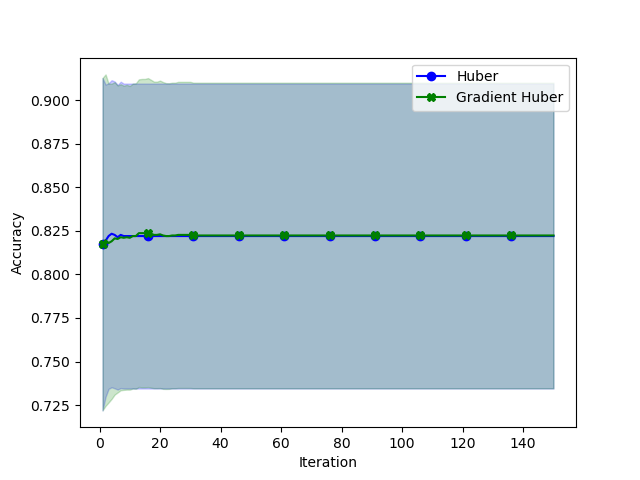}
\\
\includegraphics[scale=0.24]{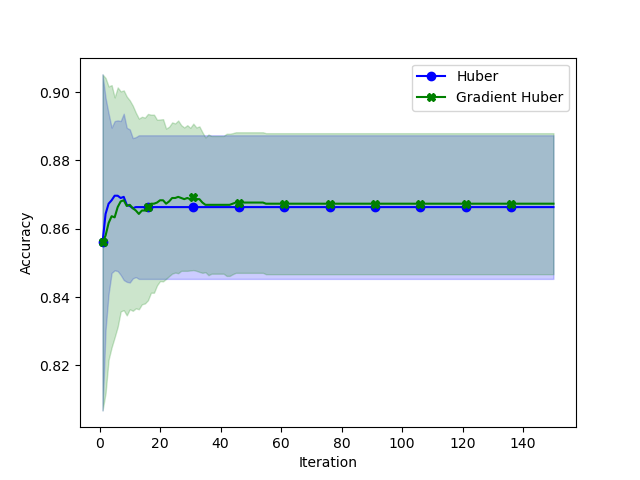}
&
\includegraphics[scale=0.24]{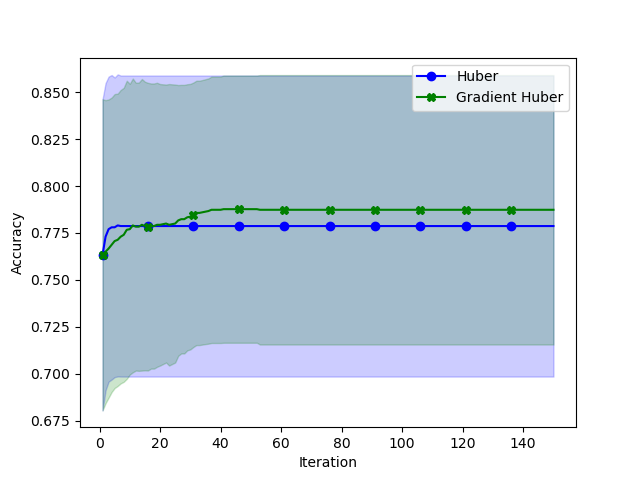}
\end{tabular}
\caption{Performance of Huber loss gradient vs the method from~\cite{Huber_clust} on Iris data. The rows correspond to percentage of noisy samples being $10\%$ and $20\%$, with columns corresponding to variance of noise being $1$ and $2$, respectively (e.g., the upper left image corresponds to $10\%$ of noisy samples, with variance $1$, etc.).}
\label{fig:iris_perc}
\end{figure}

Figure \ref{fig:perc} shows the performance of the Huber loss gradient method vs the method from~\cite{Huber_clust}, when the percentage of noisy samples and variance of noise vary. Comparing the rows, i.e., different percentage of noisy samples, we can see that both methods perform better when the percentage of noise is lower, as expected. Comparing the columns, i.e., different variance levels, we can see that our method is comparable to~\cite{Huber_clust} for variance $1$, but slightly outperforms the competing method for variance $2$. Therefore our method exhibits a similar or better performance, with a small loss in speed. However, our method provides much better convergence guarantees, as it provably converges for arbitrary initialization, while the method~\cite{Huber_clust} provides only local convergence guarantees, when already in a neighborhood of the stationary point. Figure \ref{fig:iris_perc} shows the performance of the Huber loss gradient method vs the method from~\cite{Huber_clust}, when the percentage of noisy samples and variance of noise vary. The step-size was the same as in the standard gradient $K$-means case. Comparing the rows, i.e., different percentage of noisy samples, we can see that both methods perform identically both accuracy and speed-wise, when the percentage of noisy samples is lower. However, the gradient based Huber method outperforms~\cite{Huber_clust} when the percentage of noisy samples is higher, more significantly when the variance is higher as well (bottom right image). Comparing the columns, i.e., different variance levels, we can see that both methods perform better when the variance of noise is lower.

\section{Conclusion}\label{sec:conclusion}

We proposed an approach to clustering, based on the gradient of a generic loss function, that measures clustering quality with respect to cluster assignments and cluster center positions. The approach is based on a formulation of the clustering problem that unifies the previously proposed distance based clustering approaches. The main advantage of the algorithm, compared to the standard approaches is its applicability to a wide range of clustering problems, low computational cost, as well as the ease of implementation. We prove that the sequence of centers generated by the algorithm converges to an appropriately defined fixed point, under arbitrary center initialization. We further analyze the type of fixed points our algorithm converges to, and show consistency with prior works, in case the cost is a Bregman divergence. Most notably, the assumed generic formulation includes loss functions beyond Bregman divergences (such as the Huber loss), for which the $K$-means-type averaging cluster center update step is not appropriate, 
while the step that corresponds to exact minimization with respect to the loss is computationally expensive. To combat these challenges, the proposed method involves a single gradient step with respect to the loss to update cluster centers. Numerical experiments illustrate and corroborate the results.

\bibliographystyle{IEEEtran}
\bibliography{IEEEabrv,bibliography}

\begin{thebibliography}{10}
\providecommand{\url}[1]{#1}
\csname url@samestyle\endcsname
\providecommand{\newblock}{\relax}
\providecommand{\bibinfo}[2]{#2}
\providecommand{\BIBentrySTDinterwordspacing}{\spaceskip=0pt\relax}
\providecommand{\BIBentryALTinterwordstretchfactor}{4}
\providecommand{\BIBentryALTinterwordspacing}{\spaceskip=\fontdimen2\font plus
\BIBentryALTinterwordstretchfactor\fontdimen3\font minus
  \fontdimen4\font\relax}
\providecommand{\BIBforeignlanguage}[2]{{%
\expandafter\ifx\csname l@#1\endcsname\relax
\typeout{** WARNING: IEEEtran.bst: No hyphenation pattern has been}%
\typeout{** loaded for the language `#1'. Using the pattern for}%
\typeout{** the default language instead.}%
\else
\language=\csname l@#1\endcsname
\fi
#2}}
\providecommand{\BIBdecl}{\relax}
\BIBdecl

\bibitem{banerjee_anomaly}
\BIBentryALTinterwordspacing
V.~Chandola, A.~Banerjee, and V.~Kumar, ``Anomaly detection: A survey,''
  \emph{ACM Comput. Surv.}, vol.~41, no.~3, jul 2009. [Online]. Available:
  \url{https://doi.org/10.1145/1541880.1541882}
\BIBentrySTDinterwordspacing

\bibitem{Huber_clust}
A.~K. Pediredla and C.~S. Seelamantula, ``A \text{H}uber-loss-driven clustering
  technique and its application to robust cell detection in confocal microscopy
  images,'' in \emph{2011 7th International Symposium on Image and Signal
  Processing and Analysis (ISPA)}, 2011, pp. 501--506.

\bibitem{JAIN2010651}
\BIBentryALTinterwordspacing
A.~K. Jain, ``Data clustering: 50 years beyond k-means,'' \emph{Pattern
  Recognition Letters}, vol.~31, no.~8, pp. 651--666, 2010, award winning
  papers from the 19th International Conference on Pattern Recognition (ICPR).
  [Online]. Available:
  \url{https://www.sciencedirect.com/science/article/pii/S0167865509002323}
\BIBentrySTDinterwordspacing

\bibitem{dhillon2003adivisive}
I.~S. Dhillon, S.~Mallela, and R.~Kumar, ``A divisive information-theoretic
  feature clustering algorithm for text classification,'' \emph{Journal of
  Machine Learning Research (JMLR)}, vol.~3, p. 1265–1287, Mar 2003.

\bibitem{Lloyd}
S.~Lloyd, ``Least squares quantization in \text{PCM},'' \emph{IEEE Transactions
  on Information Theory}, vol.~28, no.~2, pp. 129--137, 1982.

\bibitem{kmeans++}
D.~Arthur and S.~Vassilvitskii, ``K-means++: The advantages of careful
  seeding,'' in \emph{In Proceedings of the 18th Annual ACM-SIAM Symposium on
  Discrete Algorithms}.\hskip 1em plus 0.5em minus 0.4em\relax New Orleans,
  Louisiana: SIAM, 2007, p. 1027–1035.

\bibitem{Huang97clusteringlarge}
Z.~Huang, ``Clustering large data sets with mixed numeric and categorical
  values,'' in \emph{In The First Pacific-Asia Conference on Knowledge
  Discovery and Data Mining}, 1997, pp. 21--34.

\bibitem{arya-k-med}
V.~Arya, N.~Garg, R.~Khandekar, A.~Meyerson, K.~Munagala, and V.~Pandit,
  ``Local search heuristics for k-median and facility location problems,''
  \emph{SIAM Journal on Computing}, vol.~33, no.~3, pp. 544--562, 2004.

\bibitem{arora_kmedians}
S.~Arora, P.~Raghavan, and S.~Rao, ``Approximation schemes for euclidean
  $k$-medians and related problems,'' in \emph{Proceedings of the Thirtieth
  Annual ACM Symposium on Theory of Computing}, ser. STOC '98.\hskip 1em plus
  0.5em minus 0.4em\relax Dallas, Texas, USA: Association for Computing
  Machinery, 1998, p. 106–113.

\bibitem{4767478}
S.~Z. Selim and M.~A. Ismail, ``K-means-type algorithms: A generalized
  convergence theorem and characterization of local optimality,'' \emph{IEEE
  Transactions on Pattern Analysis and Machine Intelligence}, vol. PAMI-6,
  no.~1, pp. 81--87, 1984.

\bibitem{JMLR:v6:banerjee05b}
\BIBentryALTinterwordspacing
A.~Banerjee, S.~Merugu, I.~S. Dhillon, and J.~Ghosh, ``Clustering with bregman
  divergences,'' \emph{Journal of Machine Learning Research}, vol.~6, no.~58,
  pp. 1705--1749, 2005. [Online]. Available:
  \url{http://jmlr.org/papers/v6/banerjee05b.html}
\BIBentrySTDinterwordspacing

\bibitem{BREGMAN1967200}
L.~Bregman, ``The relaxation method of finding the common point of convex sets
  and its application to the solution of problems in convex programming,''
  \emph{USSR Computational Mathematics and Mathematical Physics}, vol.~7,
  no.~3, pp. 200--217, 1967.

\bibitem{huber_loss}
\BIBentryALTinterwordspacing
P.~J. Huber, ``{Robust Estimation of a Location Parameter},'' \emph{The Annals
  of Mathematical Statistics}, vol.~35, no.~1, pp. 73 -- 101, 1964. [Online].
  Available: \url{https://doi.org/10.1214/aoms/1177703732}
\BIBentrySTDinterwordspacing

\bibitem{liu2019robust}
\BIBentryALTinterwordspacing
C.~Liu, Q.~Sun, and K.~M. Tan, ``Robust convex clustering: How does fusion
  penalty enhance robustness?'' \emph{arXiv preprint arXiv:1906.09581}, 2019.
  [Online]. Available: \url{https://arxiv.org/abs/1906.09581}
\BIBentrySTDinterwordspacing

\bibitem{Macqueen67somemethods}
J.~MacQueen, ``Some methods for classification and analysis of multivariate
  observations,'' in \emph{In 5-th Berkeley Symposium on Mathematical
  Statistics and Probability}, no.~14.\hskip 1em plus 0.5em minus 0.4em\relax
  University of California Press, 1967, pp. 281--297.

\bibitem{bottou_kmeans}
\BIBentryALTinterwordspacing
L.~Bottou and Y.~Bengio, ``Convergence properties of the k-means algorithms,''
  in \emph{Advances in Neural Information Processing Systems}, G.~Tesauro,
  D.~Touretzky, and T.~Leen, Eds., vol.~7.\hskip 1em plus 0.5em minus
  0.4em\relax MIT Press, 1995. [Online]. Available:
  \url{https://proceedings.neurips.cc/paper/1994/file/a1140a3d0df1c81e24ae954d935e8926-Paper.pdf}
\BIBentrySTDinterwordspacing

\bibitem{Monath2017GradientbasedHC}
N.~Monath, A.~Kobren, A.~Krishnamurthy, and A.~McCallum, ``Gradient-based
  hierarchical clustering,'' in \emph{Discrete Structures in Machine Learning
  Workshop, NIPS}, Long Beach, CA, USA, 2017.

\bibitem{paul2021uniform}
\BIBentryALTinterwordspacing
D.~Paul, S.~Chakraborty, S.~Das, and J.~Xu, ``Uniform concentration bounds
  toward a unified framework for robust clustering,'' in \emph{Advances in
  Neural Information Processing Systems}, M.~Ranzato, A.~Beygelzimer,
  Y.~Dauphin, P.~Liang, and J.~W. Vaughan, Eds., vol.~34.\hskip 1em plus 0.5em
  minus 0.4em\relax Curran Associates, Inc., 2021, pp. 8307--8319. [Online].
  Available:
  \url{https://proceedings.neurips.cc/paper/2021/file/460b491b917d4185ed1f5be97229721a-Paper.pdf}
\BIBentrySTDinterwordspacing

\bibitem{cortes}
J.~Cortes, S.~Martinez, T.~Karatas, and F.~Bullo, ``Coverage control for mobile
  sensing networks,'' \emph{IEEE Transactions on Robotics and Automation},
  vol.~20, no.~2, pp. 243--255, 2004.

\bibitem{Schwager09agradient}
M.~Schwager, ``A gradient optimization approach to adaptive multi-robot
  control,'' Ph.D. dissertation, Massachusetts Institute of Technology, 2009.

\bibitem{awasthi2014center}
P.~Awasthi and M.-F. Balcan, ``Center based clustering: A foundational
  perspective,'' 2014.

\bibitem{Vattani2010TheHO}
\BIBentryALTinterwordspacing
A.~Vattani, ``The hardness of k-means clustering in the plane,'' 2009.
  [Online]. Available:
  \url{https://cseweb.ucsd.edu/~avattani/papers/kmeans_hardness.pdf}
\BIBentrySTDinterwordspacing

\bibitem{awasthi2015hardness}
\BIBentryALTinterwordspacing
P.~Awasthi, M.~Charikar, R.~Krishnaswamy, and A.~K. Sinop, ``The hardness of
  approximation of euclidean k-means,'' \emph{arXiv preprint arXiv:1502.03316},
  2015. [Online]. Available: \url{https://arxiv.org/abs/1502.03316}
\BIBentrySTDinterwordspacing

\bibitem{Kmedians-NP}
\BIBentryALTinterwordspacing
N.~Megiddo and K.~J. Supowit, ``On the complexity of some common geometric
  location problems,'' \emph{SIAM Journal on Computing}, vol.~13, no.~1, pp.
  182--196, 1984. [Online]. Available: \url{https://doi.org/10.1137/0213014}
\BIBentrySTDinterwordspacing

\bibitem{pmlr-v9-telgarsky10a}
\BIBentryALTinterwordspacing
M.~Telgarsky and A.~Vattani, ``Hartigan's method: k-means clustering without
  voronoi,'' in \emph{Proceedings of the Thirteenth International Conference on
  Artificial Intelligence and Statistics}, ser. Proceedings of Machine Learning
  Research, Y.~W. Teh and M.~Titterington, Eds., vol.~9.\hskip 1em plus 0.5em
  minus 0.4em\relax Chia Laguna Resort, Sardinia, Italy: PMLR, 13--15 May 2010,
  pp. 820--827. [Online]. Available:
  \url{https://proceedings.mlr.press/v9/telgarsky10a.html}
\BIBentrySTDinterwordspacing

\bibitem{bregman_triangle}
S.~Acharyya, A.~Banerjee, and D.~Boley, ``Bregman divergences and triangle
  inequality,'' in \emph{Proceedings of the 2013 SIAM International Conference
  on Data Mining}.\hskip 1em plus 0.5em minus 0.4em\relax SIAM, 2013, pp.
  476--484.

\bibitem{chen_bregman_metrics}
P.~Chen, Y.~Chen, and M.~Rao, ``{Metrics defined by Bregman Divergences},''
  \emph{Communications in Mathematical Sciences}, vol.~6, no.~4, pp. 915 --
  926, 2008.

\bibitem{kar2019clustering}
\BIBentryALTinterwordspacing
S.~Kar and B.~Swenson, ``Clustering with distributed data,'' \emph{arXiv
  preprint arXiv:1901.00214}, 2019. [Online]. Available:
  \url{https://arxiv.org/abs/1901.00214}
\BIBentrySTDinterwordspacing

\bibitem{Okabe2000SpatialTC}
A.~Okabe, B.~Boots, K.~Sugihara, S.~N. Chiu, and D.~Kendall, \emph{Spatial
  Tessellations: Concepts and Applications of Voronoi Diagrams, Second
  Edition}, ser. Wiley Series in Probability and Mathematical Statistics.\hskip
  1em plus 0.5em minus 0.4em\relax John Wiley \& Sons Ltd., 2000.

\bibitem{Ke2005RobustL1}
Q.~Ke and T.~Kanade, ``Robust \text{L}1 norm factorization in the presence of
  outliers and missing data by alternative convex programming,'' \emph{2005
  IEEE Computer Society Conference on Computer Vision and Pattern Recognition
  (CVPR'05)}, vol.~1, pp. 739--746, 2005.

\bibitem{lecun-mnisthandwrittendigit-2010}
\BIBentryALTinterwordspacing
Y.~LeCun, C.~Cortes, and C.~J.~C. Burges, ``{MNIST} handwritten digit
  database.'' [Online]. Available: \url{http://yann.lecun.com/exdb/mnist/}
\BIBentrySTDinterwordspacing

\bibitem{iris}
\BIBentryALTinterwordspacing
R.~A. Fisher, ``The use of multiple measurements in taxonomic problems,''
  \emph{Annals of Eugenics}, vol.~7, no.~2, pp. 179--188, 1936. [Online].
  Available:
  \url{https://onlinelibrary.wiley.com/doi/abs/10.1111/j.1469-1809.1936.tb02137.x}
\BIBentrySTDinterwordspacing

\bibitem{lectures_on_cvxopt}
Y.~Nesterov, \emph{Lectures on Convex Optimization}, 2nd~ed.\hskip 1em plus
  0.5em minus 0.4em\relax Springer Publishing Company, Incorporated, 2018.

\end{thebibliography}

\appendix

In this section we show some techinical results used in the paper. The next lemma is taken from \cite{lectures_on_cvxopt}. For the sake of completeness, we provide the proof here.

\begin{lemma}\label{lm:co-coerc}
Let $f: \mathbb{R}^d \mapsto \mathbb{R}$ be convex and have Lipschitz continuous gradients. Then, $f$ has co-coercive gradients.  
\end{lemma}
\begin{proof}
Define the function:
\begin{equation*}
    \phi_x(z) = f(z) - \big\langle \nabla f(x),z \big\rangle.
\end{equation*} It is straightforward to see that $\phi_x$ maintains convexity, for any $x \in \R^d$. It then follows that the point $x$ is a minimizer of $\phi_x$. Next, we use the following lower-bound for functions with Lipschitz continuous gradients (the proof can be found in \cite{lectures_on_cvxopt}):
\begin{equation}\label{eq:lower-b}
    \frac{1}{2L}\big\|\nabla f(x) \big\|^2 \leq f(x) - f(x^*), 
\end{equation} where $x^*$ is a minimizer of $f$. Substituting $\phi_x$ in equation (\ref{eq:lower-b}), we get
\begin{align*}
    \begin{aligned}
        \phi_x(y) - \phi_x(x) &= f(y) - \big\langle\nabla f(x),y\big\rangle - f(x) + \big\langle \nabla f(x),x\big\rangle \\ &\geq \frac{1}{2L}\|\nabla \phi_x(y)\|^2 = \frac{1}{2L}\big\| \nabla f(y) - \nabla f(x) \big\|^2. 
    \end{aligned}
\end{align*} Applying the same steps to $\phi_y$, and summing the resulting inequalities, gives the desired result. 
\end{proof}

The following lemma shows that Huber loss satisfies Assumptions \ref{asmpt:g&f}-\ref{asmpt:co-coerc}.

\begin{lemma}\label{lm:Huber}
Huber loss-based cost satisfies Assumptions \ref{asmpt:g&f}-\ref{asmpt:co-coerc}.
\end{lemma}
\begin{proof}
Note that Huber loss is an increasing function on the domain of interest, $[0,+\infty)$. By definition, 
\begin{align*}
    g(x,y) &= \| x - y \|, \\
    f(x,y) &= \phi_{\delta}(g(x,y)), 
\end{align*} hence Assumptions \ref{asmpt:g&f} and \ref{asmpt:met} are satisfied. By the same argument, for a fixed $y$, we have
\begin{equation*}
    \lim_{\|x \| \rightarrow +\infty}f(x,y) = +\infty,
\end{equation*} satisfying Assumption \ref{asmpt:coerc}. 

Next, note that $f$ is a convex function, as a composition of convex functions. By Lemma \ref{lm:co-coerc}, it suffices to show that $f(x,y)$ has Lipschitz continuous gradients. The gradient of $f$ is given by 
\begin{align*}
    \nabla_x f(x,y) = \begin{cases}
        (x-y), &\|x - y\| \leq \delta \\
        \delta\frac{x-y}{\|x-y\|}, &\|x-y\| > \delta
    \end{cases}.
\end{align*} We differentiate between the following cases:
\begin{enumerate}
    \item $\|x - y\|, \|z - y\| \leq \delta$. We then have
    \begin{equation*}
        \| \nabla f(x,y) - \nabla f(z,y) \| = \| (x - y) - (z - y)\| = \|x - z\|.
    \end{equation*}
    \item $\|x - y\| \leq \delta, \|z - y \| > \delta$ \big(the case when $\|x - y\| > \delta, \|z - y\| \leq \delta$ is analogous\big). We then have
    \begin{align*}
        \| \nabla f(x,y) &- \nabla f(z,y) \| = \Big\| (x - y) - \frac{\delta}{\|z - y\|}(z - y)\Big\| \\ &= \Big\|(x - z) + \bigg(1 - \frac{\delta}{\|z - y\|}\bigg)(z - y)\Big\| \\ &\leq \|x - z\| + \bigg(1 - \frac{\delta}{\|z - y\|}\bigg)\|z - y\| \\ &= \|x - z\| + \|z - y\| - \delta.
    \end{align*} Next, using the triangle inequality and $\|x - y\| \leq \delta$, we get
    \begin{equation*}
        \|z - y\| \leq \|x - z\| + \|x - y\| \leq \|x - z\| + \delta.
    \end{equation*} Rearranging and substituting in the equation above, we get
    \begin{equation*}
        \| \nabla f(x,y) - \nabla f(z,y) \| \leq 2\|x - z\|.
    \end{equation*}
    \item $\|x - y\|, \|z - y\| > \delta$. Without loss of generality, assume $\|x - y\| \leq \|z - y\|$. We then have
    \begin{align*}
        \| \nabla f(x,y) &- \nabla f(z,y) \| = \delta \Big\|\frac{x - y}{\|x - y\|} - \frac{z - y}{\|z - y\|}\Big\| \\ &= \delta \Big\|\frac{x - y}{\|x - y\|} - \frac{z - y}{\|z - y\|} \pm \frac{x - y}{\|z - y\|}\Big\| \\ &\leq \delta \Big(\frac{1}{\|x-y\|} - \frac{1}{\|z - y\|}\Big)\|x - y\| \\ &+ \delta\frac{\|x - z\|}{\|z - y\|} \leq \delta \frac{\|z - x\|}{\|z - y\|} + \delta \frac{\|z - x\|}{\|z - y\|} \\ &\leq 2\|z - x\|,
    \end{align*} where we use the triangle inequality and $\|x - y \| \leq \|z - y\|$ in the first inequality, while the last inequality stems from $\|z - y\| > \delta$.
\end{enumerate} Hence, we have shown that, $\forall x, y, z \in \R^d$
\begin{equation*}
    \| \nabla f(x,y) - \nabla f(z,y) \| \leq 2\|x - z\|. 
\end{equation*} By Lemma \ref{lm:co-coerc}, we see that Assumption \ref{asmpt:co-coerc} is satisfied, thus proving the claim.
\end{proof}

The following lemma shows that Jensen-Shannon divergence satisfies Assumptions \ref{asmpt:g&f}-\ref{asmpt:co-coerc}, on a properly defined support.
\begin{lemma}\label{lm:JS-div}
Let $P_{\epsilon} \subset \R^d$, for some $\epsilon > 0$, define the restricted probability simplex, i.e.
\begin{equation}\label{eq:rest_prob_simplex}
    P_{\epsilon} = \Big\{p \in \R^d: \sum_{i = 1}^d p_i = 1, \: \epsilon \leq p_i < 1 \Big\}.
\end{equation} Then, the Jensen-Shannon divergence based cost satisfies Assumptions \ref{asmpt:g&f}-\ref{asmpt:co-coerc} on $P_{\epsilon}$.
\end{lemma}
\begin{proof}
By the definition of Jensen-Shannon divergence, we have
\begin{equation*}
    D_{JS}(y \parallel x) = \frac{1}{2}D_{KL}(y \parallel m) + \frac{1}{2}D_{KL}(x \parallel m),
\end{equation*} where $m = \frac{x+y}{2}$, and $D_{KL}(\cdot \parallel \cdot)$ is the Kullback-Leibler divergence, defined by
\begin{equation*}
    D_{KL}(x \parallel y) = \sum_{i = 1}^d x_i\log \frac{x_i}{y_i}.
\end{equation*} It is shown in \cite{bregman_triangle} that the Jensen-Shannon divergence represents the square of a metric. Therefore, for 
\begin{align*}
    g(x,y) &= \sqrt{D_{JS}(y \parallel x)}, \\
    f(x,y) &= D_{JS}(y \parallel x),
\end{align*} Assumptions \ref{asmpt:g&f} and \ref{asmpt:met} are satisfied. Since the domain of interest, given by (\ref{eq:rest_prob_simplex}) is bounded, Assumption \ref{asmpt:coerc} is not of interest. 

 We next show that $D_{JS}$ is convex and has Lipschitz continuous gradients on $P_{\epsilon}$. A basic computation yields that the partial derivative of $D_{JS}$, with respect to $x_i$, is given by
\begin{equation}\label{eq:JS_part}
    \frac{\partial}{\partial x_i}D_{JS}(y \parallel x) = \frac{1}{2}\log \frac{2x_i}{x_i + y_i}.
\end{equation} It is then straightforward to see that the Hessian of $D_{JS}$ is a diagonal matrix, whose $i$-th diagonal element is given by
\begin{equation}\label{eq:DJ_hess}
    \frac{\partial^2}{\partial x_i^2}D_{JS}(y \parallel x) = \frac{y_i}{2(x_i + y_i)}.
\end{equation} Since $x, y \in P_{\epsilon}$, the expression in (\ref{eq:DJ_hess}) is positive, hence $D_{JS}$ is convex on $P_{\epsilon}$. Next, from (\ref{eq:JS_part}), for any $x,y,z \in P_{\epsilon}$, we have  
\begin{align*}
    \Big|\frac{\partial}{\partial x_i}D_{JS}(y \parallel x) - \frac{\partial}{\partial z_i}&D_{JS}(y \parallel z)\Big| =  \frac{1}{2}\Big| \log \frac{x_i}{z_i} + \log \frac{z_i + y_i}{x_i + y_i} \Big| \\ &\leq \max\Big\{\Big|\log \frac{x_i}{z_i} \Big|, \Big| \log \frac{z_i + y_i}{x_i + y_i} \Big| \Big\}.
\end{align*} Without loss of generality, assume $x_i \geq z_i$. We then have
\begin{equation*}
    \Big | \log \frac{x_i}{z_i} \Big | = \log \frac{x_i}{z_i} \leq \frac{x_i}{z_i} - 1 = \frac{x_i - z_i}{z_i} \leq \frac{1}{\epsilon}(x_i - z_i), 
\end{equation*} and
\begin{align*}
    \Big | \log \frac{z_i + y_i}{x_i + y_i} \Big | &= \log \frac{x_i + y_i}{z_i + y_i} \leq \frac{x_i + y_i}{z_i + y_i} - 1 \\ &= \frac{x_i - z_i}{z_i + y_i} \leq \frac{1}{\epsilon}(x_i - z_i), 
\end{align*} where we used $\log x \leq x - 1$ in the above inequalities. Hence, we have shown that
\begin{equation*}
    \Big|\frac{\partial}{\partial x_i}D_{JS}(y \parallel x) - \frac{\partial}{\partial z_i}D_{JS}(y \parallel z)\Big| \leq \frac{1}{\epsilon}|x_i - z_i|.
\end{equation*} By definitions of the gradient and norm, it then follows that
\begin{equation*}
    \Big\|\nabla_x D_{JS}(y \parallel x) - \nabla_z D_{JS}(y \parallel z)\Big\| \leq \frac{1}{\epsilon}\|x - z\|,
\end{equation*} which shows Lipschitz continuity of the gradients of $D_{JS}$ on $P_{\epsilon}$. Hence, by Lemma \ref{lm:co-coerc}, $D_{JS}$ satisfies Assumption \ref{asmpt:co-coerc} on $P_{\epsilon}$.
\end{proof}

\begin{remark}
Note that in general, Jensen-Shannon divergence does not satisfy Assumptions \ref{asmpt:coerc} and \ref{asmpt:co-coerc}. However, in certain problems, where the restricted probability simplex of the form (\ref{eq:rest_prob_simplex}) is a natural domain of choice, the Jensen-Shannon divergence can be applied in our framework. One such example is soft clustering under uncertainty - where no class can be ruled out with certainty, nor can a point belonging to any class be taken with certainty. Hence, for an appropriately selected $\epsilon$, the restricted probability simplex (\ref{eq:rest_prob_simplex}) represents a natural domain.
\end{remark}

\end{document}